\documentclass[11pt]{article}
\usepackage{url,ifthen}
\usepackage{srcltx}
\usepackage{multirow}
\usepackage{boxedminipage}
\usepackage[margin=1.1in]{geometry}
\usepackage{nicefrac}
\usepackage{xspace}
\usepackage{graphicx}
\usepackage{srcltx}
\usepackage{verbatim}
\usepackage[usenames]{color}
\usepackage{fullpage}
\usepackage{xspace}
\definecolor{DarkGreen}{rgb}{0.1,0.5,0.1}
\definecolor{DarkRed}{rgb}{0.5,0.1,0.1}
\definecolor{DarkBlue}{rgb}{0.1,0.1,0.5}
\usepackage[small]{caption}
\usepackage{float}
\usepackage{balance}
\usepackage{amsmath}
\usepackage{amsfonts}
\usepackage{amssymb}
\usepackage{amsthm}
\usepackage{bbm}
\usepackage{dsfont}
\usepackage{xcolor}
\usepackage[normalem]{ulem}

\newtheorem{theorem}{Theorem}[section]
\newtheorem*{namedtheorem}{\theoremname}
\newcommand{\theoremname}{testing}

\newtheorem{lemma}[theorem]{Lemma}
\newtheorem{lem}[theorem]{Lemma}

\newtheorem{observation}{Observation}
\newtheorem{corollary}[theorem]{Corollary}

\newtheorem*{question*}{Question}

\theoremstyle{definition}
\newtheorem{example}{Example}
\newtheorem{definition}[theorem]{Definition}
\newtheorem{defn}[theorem]{Definition}
\newtheorem{remark}[theorem]{Remark}

\theoremstyle{plain}
\newtheorem{Alg}{Algorithm}

\definecolor{DarkGreen}{rgb}{0.1,0.5,0.1}
\definecolor{DarkRed}{rgb}{0.5,0.1,0.1}
\definecolor{DarkBlue}{rgb}{0.1,0.1,0.5}

\usepackage[pdftex]{hyperref}
\hypersetup{
    unicode=false,          
    pdftoolbar=true,        
    pdfmenubar=true,        
    pdffitwindow=false,      
    pdfnewwindow=true,      
    colorlinks=true,       
    linkcolor=DarkGreen,          
    citecolor=DarkGreen,        
    filecolor=DarkGreen,      
    urlcolor=DarkBlue,          
}

\newcommand{\ignore}[1]{}

\renewcommand{\Pr}{\mathop{\bf Pr\/}}                    
\newcommand{\E}{\mathop{\bf E\/}}

\newcommand{\Cov}{\mathop{\bf Cov\/}}

\newcommand{\bone}{{\boldsymbol{1}}}

\newcommand{\dtv}[2]{\left\lVert #1 - #2 \right\rVert_{\mathrm{TV}}}

\makeatletter
\floatstyle{ruled}
\newfloat{fragment}{H}{lop}
\floatname{fragment}{Game}
\renewcommand{\floatc@ruled}[2]{\vspace{2pt}{\@fs@cfont \#1.\:} \#2 \par
 \vspace{1pt}}
\makeatother

\newcommand{\fnote}[1]{\textcolor{blue}{\small {\textbf{(Fred: }#1\textbf{) }}}}

\title{Learning Restricted Boltzmann Machines via Influence Maximization}
\author{Guy Bresler\thanks{
Massachusetts Institute of Technology. Electrical Engineering and Computer Science Department. Email: {\tt guy@mit.edu}. This work was supported in part by ONR N00014-17-1-2147 and NSF CCF-1565516} \and Frederic Koehler\thanks{
Massachusetts Institute of Technology. Department of Mathematics. Email: {\tt fkoehler@mit.edu}.} \and Ankur Moitra\thanks{
Massachusetts Institute of Technology. Department of Mathematics, CSAIL and IDSS. Email: {\tt moitra@mit.edu}. This work was supported in part by NSF CAREER Award CCF-1453261, NSF Large CCF-1565235, a David and Lucile Packard Fellowship, an Alfred P. Sloan Fellowship and an ONR Young Investigator Award} 
}
\begin{document} 
\maketitle
\begin{abstract}
Graphical models are a rich language for describing high-dimensional distributions in terms of their dependence structure. While there are algorithms with provable guarantees for learning undirected graphical models in a variety of settings, there has been much less progress in the important scenario when there are latent variables. Here we study Restricted Boltzmann Machines (or RBMs), which are a popular model with wide-ranging applications in dimensionality reduction, collaborative filtering, topic modeling, feature extraction and deep learning.

The main message of our paper is a strong dichotomy in the feasibility of learning RBMs, depending on the nature of the interactions between variables: ferromagnetic models can be learned efficiently, while general models cannot. In particular,
we give a simple greedy algorithm based on influence maximization to learn ferromagnetic RBMs with bounded degree. In fact, we learn a description of the distribution on the observed variables as a Markov Random Field. Our analysis is based on tools from mathematical physics that were developed to show the concavity of magnetization. Our algorithm extends straighforwardly to general ferromagnetic Ising 
models with latent variables. 

Conversely, we show that even for a contant number of latent variables with constant degree, without ferromagneticity the problem is as hard as sparse parity with noise. This hardness result is based on a sharp and surprising characterization of the representational power of bounded degree RBMs: the distribution on their observed variables can simulate any bounded order MRF. This result is of independent interest since RBMs are the building blocks of deep belief networks. 
\end{abstract}

\newpage

\section{Introduction}

\subsection{Background}

Graphical models are a rich language for describing high-dimensional distributions in terms of their dependence structure. They allow for sophisticated forms of causal reasoning and inference. Over the years, many provable algorithms for learning undirected graphical models from data have been developed, including algorithms that work on trees \cite{ChowL}, graphs of bounded treewidth \cite{KargerS}, graphs of bounded degree \cite{Bresler,Vuffray,KlivansM,Hamilton}, and under various conditions on their parameters that preclude long-range correlations \cite{RavikumarWL, BreslerMosselSly, Anandkumar}. 
In the special case of Ising models, one can learn the underlying graph in nearly quadratic time with nearly the information-theoretically optimal sample complexity~\cite{Vuffray,KlivansM}.

While all these results are for fully observed models, the presence of unobserved (or \emph{latent}) variables is of fundamental importance in a wide range of applications. Latent variable models can capture much more complex dependencies among the observed variables than fully observed models, because the variables can influence each other through unobserved mechanisms. In this way, such models allow scientific theories that explain data in a more parsimonious way to be learned and tested. They can also be used to perform dimensionality reduction \cite{RBMdimension} and feature extraction \cite{RBMfeature} and thus serve as a basis for a variety of other machine learning tasks. 

Despite their practical importance, the problem of learning graphical models with latent variables has seen much less progress. The only works we are aware of are the following: Chadrasekaran et al. \cite{ChandrasekaranPW} studied Gaussian graphical models with latent variables and sparsity and incoherence constraints. The marginal distribution on the observed variables is also a Gaussian graphical model, so it is straightforward to learn its distribution. However their focus was on discovering latent variables whose inclusion in the model ``explains away" many of the observed dependencies. Anandkumar and Valluvan \cite{AnandkumarV} were the first to give provable algorithms for learning discrete graphical models with latent variables, although they need rather strong conditions to do so. They require both that the graphical model is locally treelike and that it exhibits correlation decay.

In this paper we study  Restricted Boltzmann Machines (or RBMs), 
a widely-used class of graphical models with latent variables that were popularized by Geoffrey Hinton in the mid 2000s. In fact, our results will extend straightforwardly to general Ising models with latent variables. An RBM has $n$ observed (or visible) variables $X_1, X_2, \ldots, X_n$ and $m$ latent (or hidden) variables $Y_1, Y_2, \ldots, Y_m$ and is described by 
\begin{enumerate}

\item[(1)] an $n \times m$ interaction matrix $J$

\item[(2)] a length $n$ vector $h^{(1)}$ and a length $m$ vector $h^{(2)}$ of external fields

\end{enumerate}

\noindent Then for any $x \in \{\pm 1\}^n$ and $y \in \{\pm 1\}^m$, the probability that the model assigns to this configuration is given by:
$$\Pr(X = x, Y = y) = \frac{1}{Z}\exp\left(x^T J y + \sum_{i = 1}^n h^{(1)}_i x_i + \sum_{i = 1}^m h^{(2)}_j y_j\right) $$
where $Z$ is the partition function. It is often convenient to think about an RBM as a weighted bipartite graph whose nodes represent variables and whose weights are given by $J$. This family of models has found a number of applications including in collaborative filtering~\cite{RBMcollaborative}, topic modeling~\cite{RBMtopic} and in deep learning where they are layered on top of each other to form deep belief networks~\cite{Hinton}. As the number of layers grows, they can capture increasingly complex hierarchical dependencies among the observed variables. 

We focus on the problem of learning RBMs from i.i.d. samples of the observed variables, with particular emphasis on the practically relevant case where the latent variables have low degree. 
What makes this challenging is that even though the variables in the RBM have only pairwise interactions, when the latent variables are marginalized out we can (and usually do) get higher-order interactions. Indeed, for general graphical models with latent variables and pairwise interactions, Bogdanov, Mossel and Vadhan \cite{BogdanovMV} 
proved learning is hard (assuming $NP \ne RP$) by showing how the distribution on observed variables can simulate the uniform distribution on satisfying assignments of any given circuit. We note that this construction requires a large number (at least one for each gate) of interconnected latent variables and that the hard instances are highly complex because they come from a series of circuit manipulations.
Beyond learning, Long and Servedio \cite{LongS} proved that for RBMs  a number of other related problems 
are hard, including approximating the partition function within an exponential factor and approximate inference and sampling.

The previous work leaves the following question unresolved: 
{\em Are there natural and well-motivated families of 
Ising models with latent variables that can be efficiently learned?} 
We will answer this question affirmatively in the case of \emph{ferromagnetic} RBMs and (more generally) ferromagnetic Ising models with latent variables, which are defined as follows: A ferromagnetic RBM is one in which the interaction matrix and the vectors of external fields are nonnegative. On the other hand, we give a negative result showing that without ferromagneticity, even in the highly optimistic case when there are only a constant number of latent variables with bounded degree the problem is as hard as sparse parity with noise. This establishes a dichotomy that is just not present in the fully-observed setting.

Historically, ferromagneticity is a natural and well-studied property that plays a key role in many classic results in statistical physics and theoretical computer science.  For example, the Lee-Yang theorem \cite{lee-yang} shows that the complex zeros of the partition function of a ferromagnetic Ising model all lie on the imaginary axis \---- this property does not hold  for general Ising models. Ferromagnetic Ising models are also one of the largest classes of graphical models for which there are efficient algorithms for sampling and inference, 
which follows from the seminal work of Jerrum and Sinclair \cite{JerrumSinclair:90}. This makes them an appealing class of graphical models to be able to learn. In contrast, without ferromagneticity it is known that sampling and inference are computationally hard when the Gibbs measure on the corresponding infinite $d$-regular tree becomes non-unique \cite{SlySun}.

\subsection{Our Results}\label{sec:our-results}

First we focus on learning ferromagnetic
Restricted Boltzmann Machines with bounded degree. 
The idea behind our algorithm is simple: the observed variables that exert the most influence on some variable $X_i$ ought to be $X_i$'s two-hop neighbors.  
While this may seem intuitive, the most straightforward interpretation of this statement is false \---- the variable with the largest correlation with $X_i$ may actually be far away. In addition, even if we correct the statement (e.g. by stating instead that there should be a neighbor with large influence), such facts about graphical models are often subtle and challenging to prove. Ultimately, we make use of the famous Griffiths-Hurst-Sherman correlation inequality \cite{ghs}
to prove that the discrete influence function
\[ I_i(S) = \E \left [X_i | X_S = \{+1\}^{|S|}\right ]\]
is submodular (see Theorem~\ref{thm:submodular}). 
The GHS
inequality has found many applications in mathematical physics where it is an important ingredient in 
determining critical exponents at phase transitions. 
By recognizing that the concavity of magnetization is analogous to the properties of the multilinear extension of a submodular function \cite{multilinear}, we are able to bring to bear tools from submodular maximization to learning graphical models with latent variables.  

More precisely, we show that any set $T$ that is sufficiently close to being a maximizer of $I_i$ must contain the two-hop neighbors of $X_i$. We can thus use the greedy algorithm for maximizing a monotone submodular function \cite{greedysubmodular} to reduce our problem of finding the two-hop neighbors of $X_i$ to a set of constant size, where the constant depends on the maximum degree and upper and lower bounds on the strength of non-zero interactions. It is information theoretically impossible to learn $J$, $h^{(1)}$ and $h^{(2)}$ uniquely, but we do something almost as good and learn a description of the distribution of the observed variables as a Markov Random Field (or MRF, see Definition~\ref{dfn:MRF}):

\begin{theorem}[Informal]
There is a nearly quadratic time algorithm with logarithmic sample complexity for learning the distribution of observed variables (expressed as a Markov Random Field) for ferromagnetic Restricted Boltzmann Machines of bounded degree and upper and lower bounded interaction strength. 
\end{theorem}

\noindent See Theorem~\ref{thm:greedynbhd-works} and Theorem~\ref{thm:low-dimensional-regression} for the precise statement. We note that unlike earlier greedy algorithms for learning Ising models \cite{Bresler,Hamilton} our dependence on the maximum degree is singly exponential and hence is nearly optimal \cite{santhanamW}. In independent work, Lynn and Lee \cite{lynn2018maximizing} also considered the problem of maximizing the influence but in a \emph{known} Ising model. They gave a (conjecturally optimal) algorithm for solving this problem given an $\ell_1$-constraint on the external field.

Our algorithm extends straightforwardly to general ferromagnetic Ising models with latent variables. In this more general setting, the two-hop neighborhood of a node $i$ is replaced by an induced Markov blanket (see  Definition~\ref{dfn:blanket}), which informally corresponds to the set of observed nodes that separate $i$ from the other observed nodes. We prove:

\begin{theorem}[Informal]
There is a nearly quadratic time algorithm with logarithmic sample complexity for learning the distribution of observed variables (expressed as a Markov Random Field) for ferromagnetic Ising model with latent variables, under the conditions that the interaction strengths are upper and lower bounded, the induced Markov blankets have bounded size and that the distance between any node $i$ and any other node in its Markov blanket is bounded. 
\end{theorem}

\noindent See Theorem~\ref{thm:greedynbhd-works-general} and Theorem~\ref{thm:low-dimensional-regression-general} for the precise statement. We remark that in our setting, the maximal Fourier coefficients of the induced MRF can be arbitrarily small, which is a serious obstacle to directly applying existing algorithms for learning MRFs (see Example~\ref{ex:ferromagnetic-eta-identifiability}). Our method also has the advantage of running in near-quadratic time whereas existing MRF algorithms would require runtime $n^{d_H + 1}$, where $d_H$ is the maximum hidden degree\footnote{The induced MRF has order $d_H$, so these methods (e.g. \cite{KlivansM}) need to solve regression problems on polynomials of degree $d_H$.}. We also show how Lee-Yang properties that hold for ferromagnetic Ising models \cite{LSS} carry over to the induced MRFs in the presence of latent variables, which allows us to approximate the partition function and perform inference efficiently. See Theorem~\ref{thm:lee-yang-logz} for the precise statement. Compared to the previous settings where provable guarantees were known, ours is the first to work even when there are long range correlations.



As we alluded to earlier, being ferromagnetic turns out to be the {\em key} property in avoiding computational intractability. More precisely, we show a rather surprising converse to the well-known fact that marginalizing out a latent variable produces a higher-order interaction among its neighbors. We show that marginalizing out a collection of latent variables can produce {\em any} desired higher-order interaction among their neighbors.


\begin{theorem}[Informal]
Every binary Markov Random Field of order $d_H$ can be expressed as the distribution on observed variables of a Restricted Boltzmann Machine, where the maximum degree of any latent node is at most $d_H$. 
\end{theorem}

\noindent See Theorem~\ref{thm:mrf-as-rbm} for the precise statement. Our approach to showing the equivalence between RBMs and MRFs is to show a non-zero correlation bound between the soft absolute value function that arises from marginalizing out latent variables and a parity function. We accomplish this through estimates of the Taylor expansion of special functions. With this in hand, we can match the largest degree terms in the energy function of an MRF and recurse. 

Apart its usefulness in proving hardness, this result also resolves a basic question about the representational power of RBMs. Towards the goal of understanding deep learning, a number of recent works have shown depth separations in feed-forward neural networks \cite{telgarsky2016benefits,safran2017depth,eldan2016power}. They explicitly construct (or show that there exists) a function that can be computed by a depth $d+1$ feed-forward neural network of small size, but with depth $d$ would require exponential size. In fact, RBMs are the building block of another popular paradigm in deep learning: deep belief networks \cite{Hinton}. Towards understanding the representational power of RBMs, Martens et al. \cite{martens-rbm-representation} showed that it is possible to approximately represent the uniform distribution on satisfying inputs to the parity function, and more generally any predicate depending only on the number of 1s, using a dense RBM. 
In practice, sparse RBMs are desirable because their dependencies are easier to interpret. The above theorem exactly characterizes what distributions can be represented this way: They are exactly the bounded order MRFs.


In any case, what this means for our lower bound is that without ferromagneticity, even RBMs with a constant number of latent variables of constant degree inherits the hardness results of learning MRFs \cite{bresler2014hardness,KlivansM}, that in turn follow from the popular assumption that learning sparse parities with noise is hard. For comparison, the technique used in \cite{martens-rbm-representation} seems insufficient for this reduction \---- their method can only build certain noiseless functions.

\begin{corollary}[Informal]
If $k$-sparse noisy parity on $n$ bits is hard to learn in time $n^{o(k)}$, then it is hard to learn a representation of the distribution on $n$ observed variables (as any unnormalized function that can be efficiently computed) that is close to within total variation distance $1/3$ of a Restricted Boltzmann Machine where the maximum degree of any latent node is $d_H$ in time $n^{o(d_H)}$. This is true even if the number of hidden nodes in the RBM is promised to be constant w.r.t. $n$. 
\end{corollary}

\noindent See Theorem~\ref{thm:improperhard} for the precise statement.
Recall that it is impossible to learn the parameters of an RBM uniquely. Our result shows that learning merely a description of the distribution on the observed variables \---- i.e. a form of improper learning \---- is hard too, 
even for RBMs with only a constant number of hidden variables. In contrast, previous lower bounds were 
for  graphical models  with many more latent variables than observed variables \cite{BogdanovMV}. At the time it seemed plausible that there were large classes of graphical models with latent variables that could be efficiently learnable. But in light of how simple our hard examples are, it seems difficult to imagine any other natural and well-motivated class of graphical models with latent variables (without ferromagneticity) that is also easy to learn.

\subsection{Further Discussion} 

There is an intriguing analogy between our results and the problem of learning juntas \cite{MOS, Valiant}. While the general problem of learning $k$-juntas seems to be hard to solve in time $n^{o(k)}$ there are some special cases that can be solved much faster. Most notably, if the junta is monotone then there is a simple algorithm that works: Find all the coordinates with non-zero influence and solve the junta learning problem restricted to those coordinates. We can think of ferromagneticity as the natural analogue of monotonicity in the context of RBMs, since this property also prevents certain types of cancellations.  Are there other contraints that one can impose on RBMs, perhaps inspired by ones that work for juntas, that make the problem much easier? 

Another enticing question for future work is to study ``deeper" versions of the problem, such as ferromagnetic deep belief networks. Are there new provable algorithms for classes of deep networks to be discovered? There is a growing literature on learning deep networks under various assumptions \cite{Arora,Janzamin,Zhang, Goel}, but the ability of ferromagnetic RBMs to express long-range correlations seems to make it a potentially more challenging problem to tackle.

\section{Preliminaries}
\begin{defn} An $\emph{Ising model}$ is a probability distribution $\mu(J,h)$ on the hypercube
$\{ \pm 1\}^n$ under which
\[ \Pr(X = x) = \mu(x) = \frac{1}{Z}\exp\Big(\frac{1}{2} \sum_{i, j} J_{ij} x_i x_j + \sum_i h_i x_i\Big) \]
where $J$ is a symmetric matrix with zero diagonal
referred to as the \emph{interaction matrix}, 
$h \in \mathbb{R}^n$ is referred to as the \emph{external field} and
$Z$ is the normalizing constant known as the \emph{partition function}. 
\end{defn}
\begin{defn}
A \emph{ferromagnetic Ising model with consistent external fields} is an Ising
model such that $J_{ij} \ge 0$ for all $i,j$ and such that $h_i \ge 0$. We
will refer to this just as a \emph{ferromagnetic Ising model} from now on.
We will also refer to such a $J$ as a \emph{ferromagnetic interaction matrix}. 
\end{defn}
We are particularly interested in Ising models with \emph{hidden
  variables}; thus we introduce the well-known concept of a
\emph{Restricted Boltzmann Machine}. We will focus on the case of RBMs
in the sequel, though everything can be generalized to
Ising models with arbitrary sets of hidden nodes without much effort,
as long as there are no large connected components of hidden nodes.
\begin{defn}
Fix a vertex set $V$ which is split into two disjoint parts as $V = V_1 \cup V_2$, and let $n = |V_1|$ and $m = |V_2|$. 
A \emph{Restricted Boltzmann Machine} (or RBM) is a probability distribution on $\{\pm 1\}^n \times \{\pm 1\}^m$ under which
\[ \Pr(X = x, Y = y) = \frac{1}{Z}\exp\left(x^T J y + \sum_{i = 1}^n h^{(1)}_i x_i + \sum_{i = 1}^m h^{(2)}_j y_j\right) \]
where $J : \mathbb{R}^{n \times m}$ is the \emph{interaction matrix}, $X$ is referred to as the \emph{observed/visible nodes}, $Y$
is referred to as the \emph{latent/hidden nodes}, $h^{(1)}$ is the vector
of external fields/biases of the observed nodes and $h^{(2)}$ is the
vector of external fields for the hidden nodes. 
\end{defn}
Clearly the joint distribution of a Restricted Boltzmann Machine is just a special case of a general Ising model. Therefore we say a Restricted Boltzmann Machine is \emph{ferromagnetic} if $J_{ij} \ge 0, h^{(1)}_i \ge 0, h^{(2)}_i \ge 0$ which is consistent with our previous terminology.
\section{Submodularity of Influence in Ising models}
\begin{defn}
Fix a ferromagnetic interaction matrix $J$. We define the \emph{smooth influence function} for $X_i$ to be
\[ \mathcal{I}_i(h) = \E_{X \sim \mu(J,h)}[X_i] \]
\end{defn}
\begin{defn}
Suppose $f : \mathbb{R}_{\ge 0}^n \to \mathbb{R}$ is a $\mathcal C^2$ function, i.e. it has continuous second partial derivatives. We say that $f$ is a \emph{smooth monotone submodular function} if 
\begin{enumerate}
\item $\partial_i f \ge 0$ everywhere, and
\item $\partial_i \partial_j f \le 0$ everywhere.
\end{enumerate}
\end{defn} 
We will see that smooth monotone submodularity of $\mathcal{I}_i$ in ferromagnetic
Ising models follows from the following correlation inequality of Griffiths, Hurst and Sherman \cite{ghs}:
\begin{theorem}[GHS inequality, \cite{ghs}]\label{thm:ghs}
Let $J$ be the interaction matrix of a ferromagnetic Ising model on $n$ nodes without external field. Then for any (not necessarily distinct) $1 \le i,j,k,\ell \le n$ we have
\begin{align*} \E[X_i X_j X_k X_{\ell}] - \E[X_i X_j] \E[X_k X_{\ell}] - \E[X_i X_k] \E[X_j X_{\ell}] - \E[X_i X_{\ell}] \E[X_j X_k]& \\+ 2\E[X_i X_{\ell}] &\E[X_j X_{\ell}] \E[X_k X_{\ell}] \le 0 \,,\end{align*}
where the expectations are taken with respect to the Boltzmann distribution.
\end{theorem}
\begin{corollary}\label{corr:smooth-submodularity}
Let $J$ be a ferromagnetic interaction matrix, i.e. $J_{ij} \ge 0$. Then
for any $i \in [n]$, $\mathcal{I}_i(h) : \mathbb{R}_{\ge 0}^n \to \mathbb{R}$ is a smooth monotone submodular function.
\end{corollary}
\begin{proof}
The equivalence of correlation inequalities and partial derivative inequalities is well-known (and is used in \cite{ghs}); we include a proof only for completeness, since this precise statement does not appear in \cite{ghs}.

Let $Z(h)$ denote the partition function of the Ising model with interaction
matrix $J$ and external field $h$. Then observe that
\[ \mathcal{I}_i(h)  = \frac{\sum_x x_i \exp(x^T J x + h \cdot x)}{Z(h)} = \partial_i \log Z(h)\,, \]
so it suffices to prove that $\partial_j \partial_i \log Z(h) \ge 0$ for all $i,j$ and
$\partial_k \partial_j \partial_i \log Z(h) \le 0$ for all $i,j,k$. First observe by computing partial derivatives that
\[ \partial_j \partial_i \log Z(h) = \Cov(X_i, X_j) \ge 0\,, \]
where the covariance is taken with respect to $\mu(J,h)$ and the
inequality follows from Griffiths inequality. 
One can similarly observe that
\begin{align*} &\partial_k \partial_j \partial_i \log Z(h) =\\&\quad\E[X_i X_j X_k] - \E[X_i X_k]\E[X_j] - \E[X_i]\E[X_j X_k] - \E[X_i X_j] \E[X_k] + 2\E[X_i] \E[X_j] \E[X_k]\,, \end{align*}
where the expectation is taken with respect to $\mu(J,h)$. We now eliminate the external field by the introduction of a \emph{ghost vertex} $X_{n + 1}$ such that in the new Ising model, $J_{i (n + 1)} = h_{i}$, $J_{ij}$ is otherwise the same as before and there is no external field. In this new Ising model the marginal of $X_1, \ldots, X_n$ given $X_{n + 1} = 1$ is the same as their distribution in the first Ising model, and the marginal given $X_{n + 1} = -1$ is the same but with flipped signs. Letting $\E_{\nu}$ denote expectation with respect to this new Ising model, we see that 
\begin{align*}
&\E[X_i X_j X_k] - \E[X_i X_k]\E[X_j] - \E[X_i]\E[X_j X_k] - \E[X_i X_j] \E[X_k] + 2\E[X_i] \E[X_j] \E[X_k] \\
&= \E_{\nu}[X_i X_j X_k X_{\ell}] - \E_{\nu}[X_i X_j] \E_{\nu}[X_k X_{\ell}] - \E_{\nu}[X_i X_k] \E_{\nu}[X_j X_{\ell}] - \E_{\nu}[X_i X_{\ell}] \E_{\nu}[X_j X_k] \\&\qquad+ 2\E_{\nu}[X_i X_{\ell}] \E_{\nu}[X_j X_{\ell}] \E_{\nu}[X_k X_{\ell}]\,, 
\end{align*}
where $\ell = n + 1$. Thus it suffices to verify that this last expression is at most zero, which follows from Theorem~\ref{thm:ghs}.
\end{proof}
\begin{defn}
Fix a ferromagnetic Ising model $(J,h)$. We define the \emph{discrete influence function}
for $X_i$ to be a function from $S \subset [n]$ to $\mathbb{R}$ given by
\[ I_i(S) = \E_{X \sim \mu(J,h)}\big[X_i | X_S = \{+1\}^{|S|}\big] = \E_{X \sim \mu(J, h + \infty \boldsymbol{1}_S)}[X_i]\,. \]
\end{defn}
\begin{theorem}\label{thm:submodular}
Fix a ferromagnetic Ising model $(J,h)$.  Then for every $i$, the discrete influence function
$I_i(S)$ is a monotone submodular function.
\end{theorem} 
\begin{proof}
Since $I_i(S) = \E_{\mu(J, h + \infty \boldsymbol{1}_S)}[X_i]$, monotonicity follows immediately from Corollary~\ref{corr:smooth-submodularity}. Similarly, submodularity follows because if $S \subset T$ and we let $h_S = h + \infty\cdot \boldsymbol{1}_S$ and likewise for $h_T$, then we obtain
\[ I_i(S \cup \{j\}) - I_i(S) = \int_{h'_j = 0}^{\infty} \partial_j \mathcal{I}_i(h_S + h'_j e_j) \ge \int_{h'_j = 0}^{\infty} \partial_j \mathcal{I}_i(h_T + h'_j e_j) =  I_i(T \cup \{j\}) - I_i(T)\,, \]
where the inequality follows point-wise, by integrating the inequality $\partial_k \partial_j \mathcal{I}_i \le 0$ along any coordinate-wise non-decreasing path from $h_S + h'_j e_j$ to $h_T + h'_j e_j$.
\end{proof}
This submodularity has the following standard consequence,
which will be very useful later. 
\begin{lemma}\label{lem:good-element}
Fix a ferromagnetic Ising model $(J,h)$. Suppose $i \in [n]$ and $S,T \subset [n]$, and
 $I_i(T) > I_i(S)$. Then there exists $j \in T$ such that
\[ I_i(S \cup \{j\}) - I_i(S) \ge \frac{I_i(T) - I_i(S)}{|T \setminus S|} \]
\end{lemma}
\begin{proof}
This follows because
\[ I_i(S \cup T) - I_i(S) \ge I_i(T) - I_i(S) \]
and by submodularity, since we can go from $S$
to $S \cup T$ by adjoining elements of $T \setminus S$ one-by-one,
\[ I_i(S \cup T) - I_i(S) \le \sum_{j \in T \setminus S} I_i(S \cup \{j\}) - I_i(S) \le |T \setminus S| \max_{j \in T \setminus S} (I_i(S \cup \{j\}) - I_i(S)) \]
which completes the proof.
\end{proof}
\section{Interreducibility Between RBMs and MRFs}
First we define Markov Random Fields and introduce some standard terminology:

\begin{defn}\label{dfn:MRF}
A {\em Markov Random Field} (or MRF) of order $r$ is a probability distribution on
$\{\pm 1\}^n$ such that
\[ \Pr(X = x) = \frac{1}{Z} \exp(f(x)) \]
where $f$ is a multivariate polynomial of degree $r$ such that $f(0) = 0$, referred to as the \emph{potential}. The \emph{structure graph} of a Markov random field has vertices $1, \ldots, n$ and connects vertex $i$ and $j$ if there is a monomial in $f(x)$ with non-zero coefficient involving both $x_i$ and $x_j$.
\end{defn}
We will mostly be interested in Markov random fields of bounded degree, which we define next:
\begin{defn}\label{dfn:blanket}
Let $i$ be a vertex in a Markov random field on $\{\pm 1\}^n$, i.e. a probability
distribution of the form $\Pr(X = x) = \frac{e^{f(x)}}{Z}$. The \emph{neighborhood}
(or \emph{Markov blanket}) of $i$ is the minimal set of vertices $S$ such that
$X_i$ is conditionally independent of $X_{[n] \setminus (S \cup \{i\})}$ conditioned
on $X_S$. Equivalently\footnote{This equivalence is standard and is shown in some proofs of the Hammersley-Clifford theorem; it also follows from much more quantitative results as used in e.g. \cite{KlivansM}.}, the neighborhood is the set of vertices $j$ such that there 
exists $S \supset \{i,j\}$ and the monomial $\chi_S(x) = \prod_{k\in S}x_k$ in the Fourier expansion of $f(x)$ has non-zero coefficient. The \emph{structure graph} of an MRF is the graph on the vertices of the MRF with these prescribed neighborhoods. The \emph{degree} of the structure-graph of an MRF is the maximum degree among all vertices. 
\end{defn} 
Now we observe that the marginal distribution on the observable variables of a Restricted Boltzmann machine is a Markov
Random Field, of order at most the max degree of a hidden node. 
This is well known and was used for instance in \cite{martens-rbm-representation}, but we state and prove it for completeness:

\begin{lemma}\label{lem:rbm-as-mrf}
Fix a Restricted Boltzmann Machine $(J, h^{(1)}, h^{(2)})$. Let $w_j$
be the $j^{th}$ column of $J$, i.e. the edge weights into hidden unit $j$. Then
\[ P(X = x) = \frac{1}{Z} \exp\left(\sum_{j = 1}^m \rho(w_j \cdot x + h^{(2)}_j) + \sum_{i = 1}^n h^{(1)}_i x_i\right) \]
where $Z$ is the same as the partition function of the original RBM and
$\rho(x) = \log(e^x + e^{-x})$ (this can be thought of as a ``soft absolute value'' function).
\end{lemma}
\begin{proof}
We show a slightly more general fact. Consider a general Markov Random Field
of the form $\Pr(X = x) = \frac{1}{Z} \exp(f(x))$ where $u$ is a vertex
with only pairwise interactions, i.e.
\[ f(x) = h_u x_u + \sum_{v \sim u} w_{uv} x_{u}x_v + g(x_{\sim u}). \]
We now compute the marginal distribution on the model when $u$ is hidden. 
Observe that
\[ \Pr(X_{\sim u} = x_{\sim u}) = \exp(g(x_{\sim u})) \frac{\sum_{x_u} \exp(h_u x_u + \sum_{v \sim u} w_{uv} x_{u}x_v)}{Z}  \]
so if we let $U$ denote the neighborhood of $u$ and let
\[ f_U(x_U) = \log \sum_{x_u} \exp(h_u x_u + \sum_{v \sim u} w_{uv} x_{u}x_v) = \rho(h_u + \sum_v w_{uv} \cdot x_v) \]
where $\rho(x) = \log(e^x + e^{-x})$
then
\[ \Pr(X_{\sim u} = x_{\sim u}) = \frac{\exp(g(x_{\sim u}) + f_U(x_U))}{Z}\]
Applying this inductively gives the result of the lemma.
\end{proof}

Our main result in this section is a reduction in the other direction: We show that every MRF can be converted to an equivalent
Restricted Boltzmann Machine. This is more difficult and to our knowledge
was not known before. The key technical fact underlying the result is the following lemma,
which shows that we can build an RBM with hidden nodes connected to the observed nodes in the set 
$S$ with any desired correlation with a parity on $S$ as long as the desired correlation
is small. Then by building many of these hidden units we can capture the MRF
potential exactly.
\begin{lemma}\label{lem:parity-correlated-unit}
Fix $k \ge 0$ and let $\rho(x) = \log(e^x + e^{-x})$. Then there exist constants $\delta = \delta(k) > 0$ and 
$\gamma = \gamma(k) \in (0,\pi/2)$ such that for any $\delta'$ 
with $|\delta'| < \delta$ and $S \subset [n]$ with $|S| = k$, 
there exist $w,h$ with $|w|_1 + h \le \gamma$ such that
\[ \E_{X \sim \{\pm 1\}^n}[\rho(w \cdot X_S + h) \chi_S(X)] = \delta' \]
where the expectation is with respect to uniform measure.
\end{lemma}
\begin{proof}
This will follow by using the explicit formula for the taylor expansion of $\rho(x)$,
which we will now derive.
Recall $\rho'(x) = \tanh(x)$ and that $\tanh$ has an explicit
power series expansion with radius $\pi/2$ around 0:
\[ \tanh(x) = \sum_{n = 1}^{\infty} \frac{2^{2n}(2^{2n} - 1) B_{2n}}{(2n)!} x^{2n - 1} \]
with radius of convergence $\pi/2$, where $B_{2n} = \frac{(-1)^{n + 1} 2 (2n!)}{(2\pi)^{2n}} \zeta(2n)$ are the even Bernoulli numbers. By integrating, we see
\[ \rho(x) = \log 2 + \sum_{n = 1}^{\infty} \frac{2^{2n}(2^{2n} - 1) B_{2n}}{(2n)! (2n)} x^{2n} \]
with the same radius of convergence.

We will need the standard fact that $B_{2n} \ne 0$ for any $n \ge 1$, which follows immediately from the equation $B_{2n} = \frac{(-1)^{n + 1} 2 (2n!)}{(2\pi)^{2n}} \zeta(2n)$ and the fact that $\zeta(s) = \sum_{m = 1}^{\infty} \frac{1}{m^s} > 0$ for $s > 1$.

Now we use that the Fourier expansion of $\rho(w \cdot X_S + h)$ can be found by taking the power series expansion of $\rho$, plugging in $x = w \cdot X_S + h$ and using the identity $X_i^2 = 1$ to reduce to the parity basis. 
Let $m = \lceil \frac{|S|}{2} \rceil$ and take $\gamma \in (0,\pi/2)$. By restricting to $w,h$ such that $|w|_{1} + |h| < \gamma$ we can write
\[ \rho(w \cdot X_S + h) = \log 2 + \sum_{n = 1}^m \frac{2^{2n}(2^{2n} - 1) B_{2n}}{(2n)! (2n)} (w \cdot X_S + h)^{2n} + O(\gamma^{2m + 2}). \]
Note that in the sum, only the top $n = m$ term contributes to the coefficient of $\chi_S$.
Observe that when $|S|$ is even\footnote{We use the notation $[\chi_S] f$ to denote
the Fourier coefficient of $\chi_S$ in the Fourier expansion of $f$.},
\[ [\chi_S] (w \cdot X_S + h)^{2m} = |S|! \prod_{s \in S} w_S \]
and when $|S|$ is odd 
\[ [\chi_S] (w \cdot X_S + h)^{2m} = |S|! h \prod_{s \in S} w_S. \]
In the case where $|S|$ is even, first consider the case where $w_s = \gamma/|S|$ for $s \in S$. We then see that
\[ [\chi_S]\rho(w \cdot X_S + h) =  \frac{2^{2m}(2^{2m} - 1) B_{2m} |S|!}{(2m)! (2m) |S|^{2m}} \gamma^{2m} + O(\gamma^{2m + 2}) \]
and so as long as $\gamma$ is sufficiently small, the coefficient is positive. Next observe that if we flip the sign of $w_{s^*}$ for a single $s^* \in S$, then the sign of $[\chi_S] (w \cdot X_S + h)^{2m}$ flips and so the sign of $\rho(w \cdot X_S + h)$ must also flip when $\gamma$ is sufficiently small. Since this coefficient
varies continuously as a function of $w_{s^*}$, we see by the intermediate value
theorem we see that we can get the coefficient of $\chi_S$ to be any value in $[-\delta,\delta]$ for some $\delta > 0$.

The case when $|S|$ is odd is the same, except that we take
 $w_s = \gamma/(|S| + 1)$ and vary $h$ in $[-\gamma/(|S| + 1), \gamma/(|S| + 1)]$.
\end{proof}
\begin{theorem}\label{thm:mrf-as-rbm}
Consider an arbitrary order $r$ Markov random field
on the hypercube $\{\pm 1\}^n$, i.e. a probability distribution
of the form $\Pr(X = x) = (1/Z) \exp(f(x))$ where $f$ is a polynomial
of degree $r$. Suppose that the structure graph of the MRF has degree
$d$ and the coefficients of $f$ are bounded by a constant $M$. There is an RBM with $n$ observable nodes and parameters $(J, h^{(1)}, h^{(2)})$ with  the following properties:
\begin{enumerate}
\item[(1)] The induced MRF of the RBM equals the original MRF, i.e. the marginal law of the observed variables is the same as the distribution of the original MRF.
\item[(2)] There are at most $O_{d,M}(n)$ hidden units\footnote{This is a general upper bound; from the construction we see that if few Fourier coefficients are nonzero, then few hidden units are used.}.
\item[(3)] The degree of every vertex in the hidden layer is at most $r$.
\item[(4)] The two-hop neighborhood of every observed node equals its original MRF-neighborhood.
In particular the two-hop degree $d_{2}$ equals the degree $d$ of the structure graph of the MRF.
\end{enumerate}
\end{theorem} 
\begin{proof}
By Lemma~\ref{lem:rbm-as-mrf} this reduces to rewriting the MRF potential
in term of a summation of $\rho(\cdot)$ terms coming from hidden units.
We use the building block of Lemma~\ref{lem:parity-correlated-unit} and build
the potential of the MRF in a top-down fashion.
More precisely we can build any boolean function with Fourier mass supported on the first $r$ Fourier levels as follows:
\begin{enumerate}
\item[(a)] For each of the degree $r$ coefficients, use several copies of the parity building block to build a boolean function with the correct degree $r$ Fourier coefficients.
\item[(b)] Now recurse to the lower level coefficients --- if we use only
the building block for $|S| \le r - 1$ we will not affect the degree $r$ coefficients.
\end{enumerate}
The end result is that any Markov random field of order $r$ can be converted into a Restricted Boltzmann distribution with hidden nodes of degree at most $r$, such that the observed nodes have the same distribution as the same Markov random field. If all of the Fourier coefficients of the potential of the original MRF are bounded by $M$, then the number of hidden units we need to introduce is $O_{d,M}(n)$, because given the upper bound on $d$ each visible unit is involved in only a constant number of hyperedges, and given the upper bound on $d$ and $M$ it takes only a constant number of copies of the building block to build each Fourier coefficient.
\end{proof}

\section{The Learning Problem for RBMs}
We consider the problem of learning a Restricted Boltzmann Machine
given samples from its marginal distribution on the observed nodes $X$. Note that if we were also given samples from the joint distribution on $(X, Y)$, then this would be the standard learning problem
for Ising models as considered in e.g. \cite{Bresler,KlivansM}. However, in our setting it is impossible to recover the underlying
interaction matrix $J$ because it is not uniquely determined, i.e. Restricted Boltzmann Machines are \emph{over-parameterized} as the following examples illustrate:

\begin{example}\label{example:over-parameterized}
Consider the Restricted Boltzmann machine with two observable nodes $\{1,2\}$ and two hidden nodes labeled $\{3,4\}$ such that $J_{13} = 1, J_{23} = 1$ and $J_{14} = -1, J_{24} = 1$. Then the marginal distribution on the observables
is just independent Rademachers, so this Restricted Boltzmann machine
is not distinguishable from a model with no connections at all.
\end{example}

The previous example used non-ferromagnetic interactions
to demonstrate the over-parameterization of RBMs. However,
even when the RBM is ferromagnetic the model remains heavily
over-parameterized:

\begin{example}\label{example:over-parameterized2}
Consider a model with two observable
nodes $\{1, 2\}$, no external fields, and any number of hidden units/connections.
Since the marginal distribution on $X_1$ and $X_2$ each must be Rademacher by symmetry,
the observable distribution is specified just by a single parameter,
the covariance between $X_1$ and $X_2$. However even in the simplest case, where
there is only a single hidden unit connected to both $X_1$ and
$X_2$, there are two parameters in the model, the two edge weights
and we clearly see that these edge weights are not uniquely determined
by the distribution.
\end{example}
\begin{example}[Hidden Structure is Undetermined]\label{example:hidden-structure-undetermined}
Consider an RBM with three observable nodes $\{1,2,3\}$, a single
hidden node connected to all of them with positive edge weights, and no
external field. We know the observable distribution is an MRF so it is of the form
\[ \Pr(X = x) = \frac{1}{Z} \exp(J_{12} x_1 x_2 + J_{13} x_1 x_3 + J_{23} x_2 x_3 + J_{123} x_1 x_2 x_3). \]
Perhaps surprisingly, in this model $J_{123} = 0$. This can be seen from Lemma~\ref{lem:rbm-as-mrf} and Taylor-expanding $\rho$, or simply by symmetry: the observable distribution is symmetric under the sign flip $x \mapsto -x$ and so necessarily $J_{123} = 0$. However, since there are only
pairwise interactions in the potential it is easy to see (or we can apply Theorem~\ref{thm:mrf-as-rbm}) that there exists another RBM with only degree-$2$ hidden nodes that has exactly the same observable distribution.
\end{example}
These examples illustrate (even in restricted setting) that we cannot hope to reconstruct
$J$. Instead we consider the natural objectives from the perspective
of viewing the observable distribution as a Markov Random Field: \emph{structure learning} and learning the parameters of the Markov random field. We start with structure learning, which can be viewed as the problem of learning the \emph{two-hop neighborhoods} of the observed random variables \---- i.e. learning the square of the adjacency matrix of the bipartite structure graph. 
\begin{defn}
Suppose $i$ is an observed node in an RBM $(J,h^{(1)}, h^{(2)})$. The \emph{two-hop neighborhood} of $i$,
denoted $\mathcal{N}_2(i)$, is the smallest set $S \subset [n] \setminus \{i\}$ such that conditioned on $X_S$, $X_i$ is conditionally independent of $X_j$ for all $j \in [n] \setminus (S \cup \{i\})$.
\end{defn}

Note that $S$ is uniquely determined, because it is just the neighborhood of $i$ when we view the observable distribution as a Markov Random Field.

\begin{defn}
The \emph{two-hop degree} $d_2$ of an RBM is the maximum size of $\mathcal{N}_2(i)$ over all observed nodes $i$.
\end{defn}

Observe that $\mathcal{N}_2(i)$ is always a subset of the graph-theoretic two-hop neighborhood of $i$, i.e. the smallest set $S$ such that vertex $i$ is separated from the other observable nodes in the structure graph
of the RBM. However it may be a strict subset,
as in Example~\ref{example:over-parameterized}. We will later
show in Lemma~\ref{lem:2hop-lb} that the graph-theoretic two-hop neighborhood
always agrees with $\mathcal{N}_2(i)$ in \emph{ferromagnetic} RBMs.

In order to learn the two-hop structure of an RBM it will be necessary to have lower and upper bounds on the edge weights of the model, so we introduce the following notion of degeneracy.  This is a standard assumption in the literature on learning Ising models \cite{Bresler,Vuffray,KlivansM}. In particular, a lower bound is needed because otherwise it would be impossible to distinguish a non-edge from an edge with an arbitrarily weak interaction. An upper bound is needed to ensure the distribution of any variable is not arbitrarily close to being deterministic. 

\begin{defn}
We say that an Ising model is is \emph{$(\alpha,\beta)$\emph-nondegenerate}\footnote{Observe that the notational convention follows \cite{KlivansM} instead of \cite{Bresler}, where $\beta$ denotes the maximum edge weight.} if both:
\begin{enumerate}
\item[(1)] For every $i,j$ such that $|J_{ij}| \ne 0$, we have $|J_{ij}| > \alpha$.
\item[(2)]
$\sum_j |J_{ij}| + |h_i| \le \beta$ for every  node $i$.
\end{enumerate}
We say that an RBM is \emph{$(\alpha,\beta)$\emph-nondegenerate} if it is $(\alpha,\beta)$-nondegenerate as an Ising model.
\end{defn}
\subsection{Maximal Coefficients Can be Arbitrarily Small}\label{sec:learning-rbm-as-mrf}
In this subsection, we discuss some important obstacles to directly using regression-based methods (in particular \cite{KlivansM}) for learning the parameters of a ferromagnetic RBM. By Lemma~\ref{lem:rbm-as-mrf}, we can cast the problem of learning $\mathcal{N}_2(i)$ for each node $i$ as a structure learning problem on the induced MRF. In order to use the results of Klivans and Meka \cite{KlivansM}, we need to get bounds on the potential
\[ p(x) = \sum_{j = 1}^m \rho(w_j \cdot x + h^{(2)}_j) + \sum_{i = 1}^n h^{(1)}_i x_i. \]
In particular we need a bound on the size of the coefficients of $\partial_i p$. For a function $p : \{\pm 1\}^n \to \mathbb{R}$, let $\|p\|_1$ denote the sum of the absolute values of its Fourier coefficients. Observe that
\begin{align*} \bigg|\E_{X \sim \{\pm 1\}^n} \bigg[\partial_i p \prod_{i \in S}
    X_i\bigg]\bigg| \le \big|h^{(1)}_i\big| + \bigg|\E_{X \sim \{\pm 1\}^n}
    \bigg[\sum_{j : w_{ij} \ne 0} \rho(w_j \cdot X + h^{(2)}_j) \prod_{i
      \in S} X_i\bigg]\bigg| &\le \big|h^{(1)}_i\big| + 2\beta \mathrm{deg}(i) \\
      &\le
  2\beta(\mathrm{deg}(i) + 1) 
  \end{align*}
  which follows from Holder's inequality, since
$|\rho(w_j \cdot X + h^{(2)}_j)| \le 2\beta$ and
$|h^{(1)}_i| \le \beta$.
Furthermore the coefficient of $X_S$ in $\partial_i p$ can be non-zero only
when $S$ is a subset of the two-hop neighborhood of $i$, which follows from the Markov property.
Thus we conclude that $$\|\partial_i p\|_1 \le 2^{d_2 + 1} \beta(deg(i) + 1)$$
where $d_2$ is the maximum size of a node's two-hop neighborhood.


With this calculation in hand, the algorithm of Klivans and Meka \cite{KlivansM} is able to estimate the \emph{maximal}
Fourier coefficients\footnote{The guarantee \cite{KlivansM} for recovering non-maximal coefficients is much weaker; for why, see our Example~\ref{example:masked-external-field}.}
of the potential $p(x)$ to within $\epsilon$ additive error using roughly
$$\frac{e^{O(d_H 2^{d_2 + 1} \beta(d_V + 1))}}{\epsilon^4} \log n$$ samples where
$d_H$ is the maximum degree of any hidden node and $d_V$ is the maximum degree
of any observed node. 
We could then apply Theorem 7.2 of \cite{KlivansM} to learn the two-hop neighborhoods in the RBM if we had an additional
assumption that the induced MRF was $\eta$-identifiable:
\begin{definition}
A Markov Random Field is $\eta$-identifiable if every maximum Fourier coefficent of its potential $p$ has magnitude at least $\eta$.
\end{definition}
\noindent Unfortunately, even for MRFs induced by ferromagnetic RBMs and even under the assumption of $(\alpha,\beta)$-nondegeneracy, $\eta$ can be made to be arbitrarily small, as the following example shows:

\begin{example}[Failure of $\eta$-identifiability in ferromagnetic RBMs]\label{ex:ferromagnetic-eta-identifiability}
Consider an RBM on
  three observed nodes with spins $X_1,X_2,X_3$ and a single hidden
  node with spin $Y_1$ connected to all of the observed nodes with
  edge weight $1/4$. On the hidden node let there be an external field
  $h^{(2)}_1 = \epsilon$.  When $\epsilon = 0$, we see (as in
  Example~\ref{example:hidden-structure-undetermined}) that
  $$\Pr(X = x) = \frac{1}{Z} \exp(J X_1 X_2 + J X_1 X_3 + J X_2 X_3)$$ for
  some constant $J$ that is bounded away from zero. Hence the model is $\eta$-identifiable. However, for a small
  $\epsilon > 0$, one can see by Taylor series expansion that the
  coefficient of $X_1 X_2 X_3$ is nonzero, and by continuity it can be
  made arbitrarily small by decreasing $\epsilon$. This does not affect
  the $(\alpha,\beta)$-nondegeneracy of the model, but clearly the parameter
  $\eta$ in $\eta$-identifiability goes to zero.
\end{example}

Thus existing guarantees for regression-based methods do not seem to be strong enough for our purposes. Moreover they would even require time $n^{d_H + 1}$ to run, where $d_H$ is the hidden degree, since they solve a high-dimensional regression problem in the basis of all size $d_H$ monomials. In contrast our approach for learning the two-hop neighborhoods not only works in spite of the fact that the maximal Fourier coefficients can be arbitrarily small, it also runs in nearly quadratic time (see Theorem~\ref{thm:greedynbhd-works}).

\subsection{Hardness for Improperly Learning RBMs}
In this subsection we show that structure learning for general
(i.e. possibly non-ferromagnetic) RBMs takes time $n^{\Omega(d_H)}$
under the conjectured hardness for learning sparse parity with noise.

\begin{defn}
  The \emph{$k$-sparse parity with noise} distribution is the following
  distribution on $(X,Y)$ parameterized by a constant $\eta \in (0,1/2)$
  and an unknown subset $S$ of size at most $k$:
  \begin{enumerate}
  \item Sample $X \sim \mathrm{Unif}(\{- 1,+1\}^n)$.
  \item With probability $1/2 + \eta$, set $Y = \prod_{s \in S} X_s$,
    and with probability $1/2 - \eta$, set $Y = (-1) \prod_{s \in S} X_s$.
  \end{enumerate}
  The \emph{learning problem for $k$-sparse parity with noise} is
  to learn $S$ in polynomial time with high probability, given access
  to an oracle which generates samples of $(X,Y)$.
\end{defn}
The important point is that the joint distribution of an $(r-1)$-sparse parity with
noise $(X,Y)$ is a Markov Random Field with order $r$ interactions, and by
Theorem~\ref{thm:mrf-as-rbm} it is also the marginal distribution on the observables of an MRF with maximum hidden degree $d_H$, where the two-hop neighborhood of $Y$ is exactly the set $S$. This means if we could learn the two-hop neighborhoods of an RBM in time $n^{o(d_H)}$ this would yield a $n^{o(k)}$ algorithm for learning $k$-sparse parities with noise, which is a long-standing open question in theoretical computer science and conjectured to be impossible. The best known algorithm of Valiant \cite{Valiant} runs in time $n^{0.8k}$.  We summarize this observation in the following observation:
\begin{observation}
If $k$-sparse parity with noise on $n$ bits cannot be learned in time $n^{o(k)}$, then
  there is no algorithm which runs in time $n^{o(d_H)}$ and learns
  the two-hop neighborhood structure of a general RBM from samples
  of the distribution on its observed nodes.
\end{observation}

We will now furthermore show that this result applies even in the case of
\emph{improper learning}, where we do not aim to learn the structure
but instead aim to learn a different distribution close to the RBM.
For this purpose it is useful to recall the following equivalent\footnote{It is clear that if we have an algorithm for the learning problem,
we can use it for the hypothesis testing problem (the algorithm will
return some set $S$ and we just have to test if the parity of
$X_S$ is correlated with $Y$). In the other direction, observe that
if we pick a particular $i$ and look at the marginal distribution on
$(X_{\ne i}, Y)$ then if $i \in S$ this marginal distribution becomes
uniform on $\{ \pm 1\}^n$, whereas if $i \notin S$ this is just a sparse
parity with noise on a smaller number of variables, so if we can hypothesis
test we can efficiently determine for every $i$ whether $i$ lies in $S$.}
formulation
of learning sparse parities as a hypothesis testing problem:
\begin{defn}
  The \emph{hypothesis testing problem for $k$-sparse parity with noise}
  is to distinguish with high probability\footnote{i.e. with probability of Type I and
  Type II error going to 0 sufficiently fast.} between the cases where $(X,Y)$ is drawn
  from the uniform distribution on $\{\pm 1\}^{n + 1}$ and where
  $(X,Y)$ is drawn from the $k$-sparse parity with noise distribution for
  an unknown $S$. 
\end{defn}

We now use this to show hardness for improper learning. First we
show hardness in the case of algorithms returning a distribution $\mathcal{Q}$
with an (approximately) computable probability mass function.
\begin{theorem}
If $k$-sparse parity with noise on $n$ bits cannot be learned in time $n^{o(k)}$, then
  there is no algorithm that runs in time $n^{o(d_H)} \cdot poly(1/\epsilon)$ and returns
  a probability distribution $\mathcal{Q}$ such that:
  \begin{enumerate}
   \item[(1)] It is possible to (approximately) compute the pmf $\mathcal{Q}(x,y)$
    for $x,y \in \{\pm 1\}^n \times \{ \pm 1 \}$ in polynomial time.
  \item[(2)] $\dtv{\mathcal{Q}}{\mathcal{P}} < \epsilon$ where $\mathcal{P}$
    is the distribution on the observables of an RBM with hidden degree $d_H$.
  \end{enumerate}
\end{theorem}
\begin{proof}
  We show how to use $\mathcal{Q}$ to solve the hypothesis testing
  problem for sparse parity with noise. Recall that for any distributions $\mathcal{P}_1, \mathcal{P}_2$ 
  \[ \dtv{\mathcal{P}_1}{\mathcal{P}_2} = \E_{X \sim
      \mathcal{P}_1}\left[\frac{\mathcal{P}_1(X) -
        \mathcal{P}_2(X)}{\mathcal{P}_1(X)} \bone[\mathcal{P}_1(X) \ge
      \mathcal{P}_2(X)]\right]\] and observe that the quantity inside
  the expectation is always valued in $[0,1]$. Therefore, with $\mathcal{P}_1=\mathrm{Unif(\{\pm 1\}^{n+1})}$ and $\mathcal{P}_2=Q$, we may use $m$ samples from $\mathcal{P}_1$
and the above formula to
  approximate the TV between $\mathcal{Q}$ and the uniform
  distribution on $\{ \pm 1\}^{n + 1}$ within error $O(1/\sqrt{m})$
  with high probability (by Hoeffding's inequality). Since the TV
  distance between the uniform distribution and any particular sparse
  parity with noise is $\Omega(\eta)$ (consider the tester that looks at whether
  $Y = \prod_{s \in S} X_s$), this lets us solve the hypothesis
  testing problem for sparse parity with noise. Thus, if the algorithm can find
  $\mathcal{Q}$ in time $n^{o(d_H)}$, then this violates the conjectured
  hardness of learning sparse parity with noise.
\end{proof}
\begin{remark}
  We see from the proof of Theorem~\ref{thm:mrf-as-rbm}
  that only a constant number of hidden nodes (in terms of $n$) are used in the construction of the sparse
  parity RBM, so the above result holds even if the RBM is promised to have $O_{d_H}(1)$
  many hidden nodes.
\end{remark}
In fact, the hardness result extends even to the case when we have
access only to an \emph{unnormalized} probability distribution function.
\begin{theorem}\label{thm:improperhard}
If $k$-sparse parity with noise on $n$ bits cannot be learned in time $n^{o(k)}$, then
  there is no algorithm which runs in time $n^{o(d_H)} \cdot poly(1/\epsilon)$ and returns
  a probability distribution $\mathcal{Q}$ such that:
  \begin{enumerate}
  \item[(1)] $\dtv{\mathcal{Q}}{\mathcal{P}} < \epsilon$ where $\mathcal{P}$
    is the distribution on the observables of an RBM with hidden degree $d_H$.
  \item[(2)] There exists a function $q(x,y)$ such that
    $\mathcal{Q}(x,y) = \frac{1}{C_q} q(x,y)$ and
    $q(x,y)$ is efficiently computable. 
  \end{enumerate}
\end{theorem}
\begin{proof}
  We again reduce from the hypothesis testing problem for sparse
  parity with noise. As before suppose $Z^{(1)}, \ldots, Z^{(m)}$ are
  iid samples from the uniform distribution on $\{\pm 1\}^{n + 1}$; we
  will look at the statistics of $q(Z)$. Observe that if $\mathcal{Q}$
  were the uniform distribution, then we would have $q(Z) = C_q 1/2^{n + 1}$,
  whereas if $\mathcal{Q}$ were a sparse parity with noise we would have
  $q(Z) \propto e^{J_{\eta} \prod_{s \in S} Z_s}$ where $J_{\eta}$ is a constant
  that corresponds to $\eta$. 

  Let $q_{1/3}$ be such that the number of $z^{(i)}$ with $q(Z^{(i)}) \le q_{1/3}$ is at most $m/3$,
  and define $q_{2/3}$ similarly. Consider the quantity
  $V := \frac{q_{2/3} - q_{1/3}}{q_{1/3} + q_{2/3}}$.
  Under the uniform distribution $V$ is concentrated around zero,
  whereas under a sparse parity distribution $V$ is concentrated
  about $\frac{e^{J_{\eta}}  - e^{-J_{\eta}}}{e^{J_{\eta}} + e^{-J_{\eta}}}$.
  The same is true under distributions which are close in TV to either
  distribution, since $V$ is defined in terms of cumulative distribution function statistics. Therefore
  we can distinguish between independent bits and sparse parity with
  noise efficiently given access to $q$.
\end{proof}
\section{A Greedy Algorithm for Learning Ferromagnetic RBMs}
We describe a simple and efficient greedy algorithm for learning the two-hop neighborhood of an observed node $i$ from samples, if the RBM is ferromagnetic. This algorithm is much faster than is possible
for general RBMs according to the lower bound of the previous subsection.
Let $\widetilde{\E}$ denote the empirical
expectation, and define the \emph{empirical influence}
$$\widetilde{I_i}(S) = \widetilde{\E}[X_i | X_S = \{1\}^S]\,.$$ Let $\eta > 0$ be a real-valued parameter and $k \ge 1$ an integer parameter to be specified later.

\begin{center}
\noindent\rule{16cm}{0.4pt}
\end{center}
\vspace{-.2cm}
$\quad$ \textbf{Algorithm 1:} {\sc GreedyNbhd($i$)}
\vspace{-.3cm}
\begin{center}
\noindent\rule{16cm}{0.4pt}
\end{center}
\begin{enumerate} \itemsep 0pt
\small
\item Set $S_0 := \emptyset$.
\item For $t$ from 0 to $k - 1$:
  \begin{enumerate}
    \item Let $j_{t + 1} := \arg\max_j \widetilde{I}_i(S_t \cup \{j\})$, where $j$ ranges over all observed nodes.
    \item Set $S_{t + 1} := S_t \cup \{j_{t + 1}\}$
  \end{enumerate}
\item Let $\widetilde{\mathcal{N}}_2 := \{j \in S_k : \widetilde{I}_i(S_k) - \widetilde{I}(S_k \setminus \{j\}) \ge \eta \}$.
\item Return $\widetilde{\mathcal{N}}_2$.
\end{enumerate}

\vspace{-.5cm}

\begin{center}
\noindent\rule{16cm}{0.4pt}
\end{center}

\vspace{.1cm}

$(\alpha,\beta)$-nondegeneracy has the following useful consequences:
\begin{lemma}\label{lem:spin-freedom}
Suppose $X_i$ is the spin at vertex $i$ in an $(\alpha,\beta)$-nondegenerate Ising model. Then $\min(\Pr(X_i = +),\Pr(X_i = -)) \ge \sigma(-2\beta)$, where 
$\sigma(x) = \frac{1}{1 + e^{-x}}$.
\end{lemma}
\begin{proof}
We show the lower bound for $\Pr(X_i = +)$ since the two cases are symmetrical.
By the law of total expectation, it suffices to show that for any fixing $x_{\ne i}$
of the other spins $X_{\ne i}$ that $\Pr(X_i = + | X_{\ne i} = x_{\ne i}) \ge \sigma(-2 \beta)$,
and this follows because
\[ \Pr(X_i = + | X_{\ne i} = x_{\ne i}) = \frac{\exp(\sum_{j : j \ne i} J_{ij} x_j)}{\exp(\sum_{j : j \ne i} J_{ij} x_j) + \exp(-\sum_{j : j \ne i} J_{ij} x_j)} = \sigma\Big(2 \sum_{j : j \ne i} J_{ij} x_j\Big) \ge \sigma(-2 \beta). \qedhere\]
\end{proof}
\begin{lemma}\label{lem:tanh-sensitivity}
Suppose $X_i$ is the spin at vertex $i$ in an $(\alpha,\beta)$-nondegenerate Ising model and $j$ is a neighbor of $i$. Then for any fixing $x_{\ne i,j}$ of the other spins $X_{i \ne j}$ of the Ising model, we have
\[ \big|\E[X_i | X_j = 1, X_{\ne i,j} = x_{\ne i,j}] - \E[X_i | X_j = -1, X_{\ne i,j} = x_{\ne i,j}]\big| \ge 2\alpha (1 - \tanh^2(\beta))\,. \]
\end{lemma}
\begin{proof}
Observe that
\[ \E[X_i | X_{\ne i}] = \tanh\Big(\sum_{k : k \ne i} J_{ik} x_k\Big). \]
Since $\tanh'(x) = 1 - \tanh^2(x)$ and $\tanh$ is a monotone function, we see that
if we let $x = -J_{ij} + \sum_{k : k \notin \{i,j\}} J_{ik} x_k$, then since $x \in [-\beta,\beta]$ we have
\[ |\tanh(x + 2J_{ij}) - \tanh(x)| \ge 2|J_{ij}| \inf_{x \in [-\beta,\beta]} (1 - \tanh^2(x)) \ge 2\alpha (1 - \tanh^2(\beta))\,.\qedhere\]
\end{proof}
The following lemma shows quantitatively that in a nondegenerate ferromagnetic RBM, the graph-theoretic two-hop neighborhood of a vertex $i$ always equals $\mathcal{N}_2(i)$, the two-hop Markov blanket. It is immediate from the Markov property for the RBM as an Ising model that $\mathcal{N}_2(i)$ is contained in the graph-theoretic two-hop neighborhood, and the lemma implies the reverse inclusion. 
\begin{lemma}\label{lem:2hop-lb}
Suppose node $i$ is an observed node in a ferromagnetic $(\alpha,\beta)$-nondegenerate RBM and denote by $T$ the graph-theoretic two-hop neighborhood of $i$. If
$S \subset [n]$ is a set of nodes such that $T \not\subset S$, then for any $j \in T \setminus S$, we have
\[ I_i(S \cup \{j\}) - I_i(S) \ge 2\alpha^2 \sigma(-2 \beta)(1 - \tanh(\beta))^2 \,. \]
\end{lemma}
\begin{proof}
Fix $j \in \mathcal{N_2}(i) \setminus S$ and let $k$ be a hidden
node which is a mutual neighbor of $i,j$. Now observe by submodularity
it suffices to prove the lower bound when $S = [n] \setminus \{i,j,k\}$.
Then
\begin{align*}
I_i(S \cup \{j\}) - I_i(S)
&= \E[X_i | X_{S} = 1^S, X_j = 1] - \E[X_i | X_S = 1^S] \\
&= \E[X_i | X_{S} = 1^S, X_j = 1] - \E[X_i | X_S = 1^S, X_j = 1]\Pr(X_j = 1 | X_S = 1^S) \\
&\quad - \E[X_i | X_S = 1^S, X_j = -1]\Pr(X_j = -1 | X_S = 1^s) \\
&= \Pr(X_j = -1 | X_S = 1^S)(\E[X_i | X_{S} = 1^S, X_j = 1] - \E[X_i | X_S = 1^S, X_j = -1]) \\
&\ge \sigma(-2\beta)(\E[X_i | X_{S} = 1^S, X_j = 1] - \E[X_i | X_S = 1^S, X_j = -1])\,.
\end{align*}
Furthermore when $S = [n] \setminus \{i,j,k\}$ we know that $X_i$ and $X_j$ are independent conditioned on $k$, so
\begin{align*}
&\E[X_i | X_{S} = 1^S, X_j = 1] - \E[X_i | X_S = 1^S, X_j = -1] \\
&= \E[X_i | X_S = 1^S, X_k = 1](\Pr(X_k = 1 | X_{S} = 1^S, X_j = 1]  - \Pr(X_k = 1 | X_{S} = 1^S, X_j = -1]) \\
&\quad + \E[X_i | X_S = 1^S, X_k = -1](\Pr(X_k = -1 | X_{S} = 1^S, X_j = 1] - \Pr(X_k = -1 | X_{S} = 1^S, X_j = -1]) \\
&= \E[X_i | X_S = 1^S, X_k = 1](\Pr(X_k = 1 | X_{S} = 1^S, X_j = 1]  - \Pr(X_k = 1 | X_{S} = 1^S, X_j = -1]) \\
&\quad - \E[X_i | X_S = 1^S, X_k = -1](\Pr(X_k = 1 | X_{S} = 1^S, X_j = 1] - \Pr(X_k = 1 | X_{S} = 1^S, X_j = -1]) \\
&= (\E[X_i | X_S = 1^S, X_k = 1] - \E[X_i | X_S = 1^S, X_k = -1]) \\
& \quad \cdot (\Pr(X_k = 1 | X_{S} = 1^S, X_j = 1] - \Pr(X_k = 1 | X_{S} = 1^S, X_j = -1]) \\
&\ge 2\alpha^2 (1 - \tanh(\beta))^2\,,
\end{align*}
where the last inequality is by Lemma~\ref{lem:tanh-sensitivity}.
\end{proof}
As the first step in analyzing our algorithm, we first determine
a sufficient number of samples to compute $\widetilde{I}_i(S)$ to a specified
precision for all small sets $S$.
\begin{lemma}\label{lem:influence-sample-complexity}
Let $\delta,\epsilon > 0$ and $k \ge 0$. If we are given $M$ samples from a ferromagnetic Restricted Boltzmann Machine and $M \ge 2^{2k + 1} (1/\epsilon^2) (\log(n) + k\log(en/k)) \log(4/\delta)$, then with probability at least $1 - \delta$,
for all $S \subset [n]$ such that $|S| \le k$
\[ |I_i(S) - \widetilde{I}_i(S)| < \epsilon. \]
\end{lemma}
\begin{proof}
First observe that 
\[ \Pr(X_S = 1^S) \ge 2^{-|S|} \]
because in a ferromagnetic model (which by our definition has nonnegative external fields), $X_S = 1^S$ is the most likely state to observe for $X_S$. This inequality can also be proved by applying Griffith's inequality iteratively. Also observe
that the total number of sets $S$ we consider is  $\sum_{j = 0}^k {n \choose j} \le (en/k)^k$.
For each $S$, let $M_S$ be the number of samples where $X_S = 1^S$. Then
by Hoeffding's inequality,
\[ \Pr(M_S - \E M_S < - t) \le e^{-2t^2/M}. \]
In particular, since $\E M_S \ge 2^{-k} M$ as long as $|S| \le k$,
\[ \Pr(M_S < 2^{-k - 1} M) \le e^{-2M 2^{-2k - 2}} \]

Now by the usual rejection sampling argument, those samples which
have $X_S = 1^S$ are independent and identically distributed samples from the conditional law. 
(One way to see this is that we can think of each sample as equivalently
being generated by first sampling $X_S$, then sampling the rest of the spins
conditioned on $X_S$). 
Therefore,
by another application of Hoeffding's inequality, for a particular
choice of $i,S$ we have
\[ \Pr(|\widetilde{I}_i(S) - I_i(S)| \ge \epsilon | M_S) \le 2 e^{-2M_S \epsilon^2}\,. \]
Now by the law of total expectation
\begin{align*}
\Pr(|\widetilde{I}_i(S) - I_i(S)| \ge \epsilon) 
&= \E[\Pr(|\widetilde{I}_i(S) - I_i(S)| \ge \epsilon | M_S)] \\
&\le 2\E[e^{-2M_S \epsilon^2}] \\
&= 2\E[(\bone_{M_S < 2^{-k - 1} M} + \bone_{M_S \ge 2^{-k - 1} M}) e^{-2M_S \epsilon^2}] \\
&\le 2e^{-2M 2^{-2k - 2}} + 2e^{-2 (2^{-k - 1} M) \epsilon^2} \\
&\le 4e^{-M 2^{-2k - 1} \epsilon^2}\,.
\end{align*}
And by the union bound, the probability that $|\widetilde{I}_i(S) - I_i(S)| \ge \epsilon$ for some $i,S$ is at most 
\[ n (en/k)^k 4e^{-M 2^{-2k - 1} \epsilon^2}\,. \]
Therefore if we take $M \ge 2^{2k + 1} (1/\epsilon^2) (\log(n) + k\log(en/k)) \log(4/\delta)$ the result follows.


\end{proof}
We also analyze the standard greedy algorithm for submodular maximization
under noise; this corresponds to Steps 1-2 of the algorithm.
\begin{lemma}\label{lem:noisy-greedy}
  Suppose $t \ge 0$ is an integer, $f(S)$ is a monotone
  submodular function and $\widetilde{f}(S)$ is an approximation to $f$
  such that $|f(S) - \widetilde{f}(S)| < \epsilon$ for some uniform
  $\epsilon > 0$ and all $S$ such that $|S| \le t$.  Let
  $S_0 = \emptyset$ and suppose $S_{i + 1}$ is formed by greedily
  adding to $S_i$ the element $j$ which maximizes
  $\widetilde{f}(S_{i} \cup \{j\})$. Then for any set $T$, we have
\[ f(T) - f(S_t) \le (1 - 1/|T|)^t f(T) + |T| \epsilon\,. \]
\end{lemma}
\begin{proof}
Consider going from $S_t$ to $S_{t + 1}$.
By Lemma~\ref{lem:good-element}, there exists some $j^*$ such that
\[ f(S_t \cup \{j^*\}) - f(S_t) \ge \frac{f(T) - f(S_t)}{|T|}\,. \]
Therefore for the $j$ which is chosen to form $S_{t + 1}$, we know
\[  (f(T) - f(S_t)) - (f(T) - f(S_{t + 1})) = f(S_{t + 1}) - f(S_t) = f(S_t \cup \{j\}) - f(S_t) \ge \frac{f(T) - f(S_t)}{|T|} - \epsilon\,. \]
Rearranging, we see that
\[ f(T) - f(S_{t + 1}) \le (1 - 1/|T|)(f(T) - f(S_t)) + \epsilon \]
and the result follows by iterating this inequality (note that the sum of the epsilon terms forms a geometric series).
\end{proof}
\begin{theorem}\label{thm:greedynbhd-works}
  Let $\delta > 0$. Suppose $X^{(1)}, \ldots, X^{(M)}$ are samples from the observable distribution of a ferromagnetic Restricted Boltzmann machine which is $(\alpha,\beta)$-nondegenerate, and has two-hop degree $d_2$. Then if
  \[ M \ge 2^{2k + 3} (d_2/\eta)^2 (\log(n) + k\log(en/k)) \log(4/\delta) \]
  where we set
  \[ \eta = \alpha^2 \sigma(-2\beta)(1 - \tanh(\beta))^2, \qquad k = d_2 \log(4/\eta), \] for every $i$ algorithm {\sc GreedyNbhd} returns $\mathcal{N}_2(i)$, with probability at least $1 - \delta$. Furthermore the total runtime is $O(M k n^2) = e^{O(\beta d_2 - \log(\alpha))} n^2 \log(n)$.
\end{theorem}
\begin{proof}
  Apply Lemma~\ref{lem:influence-sample-complexity} with
  $\epsilon = \eta/(4d_2)$; then for our choice of $M$ we
  have that $|\widetilde{I}_i(S) - I_i(S)| < \eta/(4 d_2)$ for all
  $S$ with $|S| \le k$. Then applying Lemma~\ref{lem:noisy-greedy}
  and using our choice of $k$ with the inequality $1 + x \le e^x$,
  we have
  \begin{equation}\label{eqn:small-excess}
  I_i(\mathcal{N}_2(i)) - I_i(S_k) \le (1 - 1/d_2)^{k} + \eta/4 \le \eta/2.
  \end{equation}
  Suppose $S_k$ does not contain the two-hop neighborhood
  of $i$. then we can take any of the two-hop neighbors
  $j \in \mathcal{N}_2(i) \setminus S_k$ and see that
  \[ I_i(\mathcal{N}_2(i)) - I_i(S_k) \ge I_i(S_k \cup \{j\}) -
    I_i(S_k) \ge 2\alpha^2 \sigma(-2\beta)(1 - \tanh(\beta))^2 = 2\eta \]
  where the first inequality follows since $\mathcal{N}_2(i)$ is the
  global maximizer of $I_i$ among all subsets of the observed nodes (by
  monotonicity and the Markov property), and the second inequality is Lemma~\ref{lem:2hop-lb}. This contradicts \eqref{eqn:small-excess}, therefore $S_k$ does
  contain the entire two-hop neighborhood of $i$.

  It remains to show that Step 3 of the algorithm leaves in
  $\widetilde{\mathcal{N}}_2$ exactly the elements of $S$ which are in the
  two-hop neighborhood. Since $|\widetilde{I}_i(S) - I_i(S)| < \eta/(4 d_2)$
  for every set $S$ with $|S| \le k$, this is straightforward: if $j$
  is a two-hop neighbor, then by Lemma~\ref{lem:2hop-lb} and triangle
  inequality we see that
  $$|\widetilde{I}_i(S_k) - \widetilde{I}_i(S_k \setminus \{j\})| \ge 2\eta -
  \eta/2 > \eta$$ If $j$ is not a two-hop neighbor, then
  $I_i(S_k) - I_i(S_k \setminus \{j\}) = 0$ by the Markov property, so
  by triangle inequality
  $|\widetilde{I}_i(S_k) - \widetilde{I}_i(S_k \setminus \{j\})| \le \eta/2 <
  \eta$. Thus for each $i$, the returned
  $\widetilde{\mathcal{N}}_2$ is the true two-hop neighborhood
  of vertex $i$.

To analyze the runtime, observe that the loop goes through at most $k$
steps, and each iteration of the loop takes time $O(nM)$ to consider
each $j$ and compute $\widetilde{I}(S_t \cup \{j\})$ from samples, and
we run {\sc GreedyNbhd} from each of the $n$ vertices.
\end{proof}
\subsection{Improving the Sample Complexity}
We consider the following algorithm for learning the two-hop neighborhood of an RBM, which is inspired by the approach of \cite{BreslerMosselSly} for learning Ising models and MRFs (without hidden nodes). As we will show this algorithm has better sample complexity than the previous one, but sacrifices speed in order to achieve this: it runs in time $O(n^{d_2 + 1} \log(n))$. This leaves open the question of whether there is a \emph{statistical-computational gap} inherent in the RBM-learning problem. As before, $\eta > 0$ is a parameter we will specify later.

\begin{center}
\noindent\rule{16cm}{0.4pt}
\end{center}
\vspace{-.2cm}
$\quad$ \textbf{Algorithm 2:} {\sc SearchNbhd($i$)}
\vspace{-.3cm}
\begin{center}
\noindent\rule{16cm}{0.4pt}
\end{center}
\begin{enumerate} \itemsep 0pt
\small
\item Let $\mathcal{F}$ be the family of subsets of $n$ of size at most $d_2$ such that $S \in \mathcal{F}$ when for every $j$, 
\[\widetilde{I}_i(S \cup \{j\}) - \widetilde{I}_i(S) \le \eta . \]
\item Return $\arg\min_{S \in \mathcal{F}} |S|$.
\end{enumerate}

\vspace{-.5cm}

\begin{center}
\noindent\rule{16cm}{0.4pt}
\end{center}

\vspace{.1cm}

\begin{theorem}\label{thm:searchnbhd-works}
Algorithm {\sc SearchNbhd} returns the correct neighborhood
with probability at least $1 - \delta$ given 
  \[ M \ge 2^{2 d_2 + 3} (1/\eta)^2 (\log(n) + d_2\log(en/d_2)) \log(4/\delta) \]
samples, when $\eta = \alpha^2 \sigma(-2\beta)(1 - \tanh(\beta))^2$. The algorithm runs in time
$O(n^{d_2 + 1} M)$.
\end{theorem}
\begin{proof}
Apply Lemma~\ref{lem:influence-sample-complexity} with
  $\epsilon = \eta/4$ and $k = d_2$; then for our choice of $M$ we
  have with probability at least $1 - \delta$ that $|\widetilde{I}_i(S) - I_i(S)| < \eta/4$ for all
  $S$ with $|S| \le d_2$. Then, as in the proof of 
  Theorem~\ref{thm:greedynbhd-works} we can apply the
  triangle inequality and Lemma~\ref{lem:2hop-lb} to show 
  that $\mathcal{F}$ contains only supersets of the two-hop
  neighborhood, and that $\mathcal{N}_2$ lies in $\mathcal{F}$;
  hence $\mathcal{N}_2$ is the unique smallest set in     $\mathcal{F}$ and so
  the output of {\sc SearchNbhd($i$)} is correct for every $i$.
\end{proof}
Note that the sample complexity is $e^{O(\beta + d_2 - \log \alpha)} \log n$.
This straightforwardly implies a bound for the special case of learning Ising models of bounded degree $d$ without hidden nodes (which can be built as RBMs using a single vertex for each edge of the original model) which also has sample complexity $e^{O(\beta + d - \log \alpha)} \log n$ in terms of the edge weights of the original Ising model. Then we see by the result of \cite{santhanamW} that for the special case of learning Ising models, this algorithm is essentially information-theoretically optimal (up to constants).

\subsection{Learning the Induced MRF via Regression}
Once the two-hop neighborhoods of the observed nodes in the RBM are determined, it becomes much easier to learn the potential (i.e.  the hyper-edge weights) of the induced MRF; this is because the problem of predicting $X_i$ based on the other spins goes from being a high-dimensional regression problem in $O(n^{d_H})$ monomials to a low-dimensional problem, since we can restrict to the monomials supported on the actual two-hop neighborhood. We will show how this lets us get much
better results for recovering the MRF potential in our setting: compare  Theorem~7.5 of \cite{KlivansM} we have changed the norm from the 1-norm to infinity-norm, but in return have reduced the run-time to $O(n^2 M)$ instead of $O(n^r M)$, and similarly reduced the sample complexity (in terms of $n$) from $O(n^r)$ to $O(\log n)$, an exponential improvement. Note that just changing the 1-norm to infinity-norm, without
providing the extra neighborhood information, \emph{does not} suffice
for getting the guarantee below\footnote{Note that without the neighborhood information, the coefficients of the maximal monomials can be easily recovered to $\epsilon$-error without the neighborhood information, but not the lower-order monomials --- their coefficients are harder to learn in general. See also Example~\ref{example:masked-external-field}.}.

As in \cite{KlivansM} we will learn the MRF potential
by performing a series of regressions, to predict each
of the $X_i$ in terms of the other spins. Here we could use the regression guarantee from Theorem 3.1 of \cite{KlivansM}, or standard guarantees for logistic regression, but the guarantee of the {\sc GLMTron} algorithm of \cite{glmtron} also works and is slightly more convenient to use. We cite only the special case of the general guarantee which we use.
\begin{theorem}[Theorem 1 of \cite{glmtron}]\label{thm:glmtron}
Suppose that $X$ is a random variable such that
$\|X\|_2 \le R_1$ almost surely and suppose that
$\|w^*\|_2 \le R_2$. Suppose $Y$ is a random variable valued in $[-1,1]$ such that $\E[Y | X] = \tanh(w^* \cdot X)$. Then there is a polynomial time algorithm ({\sc GLMTron} with hold-out set validation) which given $R_1,R_2$ and $m$ samples of $(X,Y)$, with probability at least $1 - \delta$ finds a $w$ such that
\[ \E\big[(\sigma(w^* \cdot X) - \sigma(w \cdot X))^2\big] = O\Big(R_1 R_2 \sqrt{\frac{\log(m/\delta)}{m}}\Big) \]
\end{theorem}
We use this regression guarantee to derive
a corresponding guarantee for learning the underlying weights via Fourier analysis.
\begin{lemma}\label{lem:risk-to-parameters}
Suppose that $X$ is a random variable valued
in $\{\pm 1\}^n$ and there exists $\delta > 0$ such that for any $x$, $\Pr[X = x] \ge \delta/2^n$.
Suppose that $f,g : \{\pm 1\}^n \to \mathbb{R}$ and write the Fourier expansion of $f$ as $f(x) = \sum_{S \subset [n]} \hat{f}(S) \prod_{s \in S} x_s$.
Then if we view $\hat{f}, \hat{g}$ as vectors of coefficients, we have
\[ \|\hat{f} - \hat{g}\|_2^2 \le \frac{1}{\delta}\E[(f(X) - g(X))^2] \]
\end{lemma}
\begin{proof}
Observe that we can decompose the distribution
of $X$ into a mixture $\delta \mathcal{P}_1 + (1 - \delta) \mathcal{P}_2$ where $\mathcal{P}_1$ is the uniform distribution on $\{ \pm 1\}^n$ and $\mathcal{P}_2$ is some other distribution. Therefore
\[ \E[(f(X) - g(X))^2] \ge \delta \E_{X \sim \mathcal{P}_1}[(f(X) - g(X))^2] = \delta \|\hat{f} - \hat{g}\|_2^2 \]
 where the last equality is Parseval's theorem.
\end{proof}
We can now show that the natural algorithm which uses {\sc GreedyNbhd} to learn the two-hop neighborhoods of the MRF, combined with running the {\sc GLMTron} algorithm of \cite{glmtron} within each of these neighborhoods to learn the coefficients of the MRF potential, successfully reconstructs the potential $p^*$. Here we adopt the convention
that $p^*(\emptyset) = 0$ since there is ambiguity in the constant term of the potential.
\begin{theorem}\label{thm:low-dimensional-regression}
Consider an unknown $(\alpha,\beta)$-nondegenerate ferromagnetic Restricted Boltzmann Machine with two-hop degree $d_2$. Let $p^*(x)$ be the potential of the MRF induced on the observed nodes.
There exists an algorithm, which given $\alpha,\beta,d_2$,
with probability at least $1 - \delta$ 
finds a polynomial $p$ of degree at most $d_2$ such that 
\[ \|\hat{p} - \hat{p^*}\|_{\infty}^2 = O\left(\frac{\beta \sigma(-2\beta)^{d_2}}{2^{d_2/2}(1 - \tanh^2(\beta))^2} \sqrt{\frac{\log(Mn/\delta)}{M}}\right) \]
given $M$ samples from the the distribution on the observed nodes, provided that $M$ is at the required $M$ in Theorem~\ref{thm:greedynbhd-works}.
\end{theorem}
\begin{proof}
By the last assumption, we can apply Algorithm {\sc GreedyNbhd} and the analysis of Theorem~\ref{thm:greedynbhd-works} to reconstruct the two-hop neighborhoods $\mathcal{N}_2(i)$ for all $i$ with high probability. We now proceed to show how to reconstruct the MRF potential given knowledge of these two-hop neighborhoods; the result will then follow by taking a union bound over these two steps.

Fix a node $i$. We consider a kernel regression (using the {\sc GLMTron} algorithm) to predict $X_i$ given the nodes in its two-hop neighborhood. Let $p^*_i = \partial_i p^*$ and observe this is a polynomial only in $X_{\ne i}$.
Observe that $\E[X_i | X_{\ne i}] = \tanh(p^*_i(X_{\ne i}))$
and furthermore by Lemma~\ref{lem:spin-freedom} that $|p^*_i(X_{\ne i})| \le \beta$ always. 
Therefore by Parseval's theorem, $\|\hat{p}^*_i\|_2 \le \beta$. 

We use the guarantee of Lemma~\ref{lem:risk-to-parameters}, applied to $X' = (\prod_{s \in S} x_S)_{S \subset \mathcal{N}_2(i)}$ and $Y = X_i$, to show that the returned $p_i$ satisfies
\[ \E[(\tanh(p_i^*(X)) - \tanh(p_i(X)))^2] = O\left(\beta 2^{d_2/2} \sqrt{\frac{\log(M/\delta)}{M}}\right)\]
with probability at least $1 - \delta$. Since the derivative of $\tanh$ on $[-\beta,\beta]$ is lower bounded by $1 - \tanh^2(\beta)$ and $\tanh$ is monotonically increasing, this implies
\[ \E[(p_i^*(X) - p_i(X))^2] = O\left(\frac{\beta 2^{d_2/2}}{(1 - \tanh^2(\beta))^2} \sqrt{\frac{\log(M/\delta)}{M}}\right) \]
Finally by iterating the argument from the proof of Lemma~\ref{lem:spin-freedom} we see we can apply Lemma~\ref{lem:risk-to-parameters} with $\delta = 2^{d_2} \sigma(-2 \beta)^{d_2}$ to conclude
\[ \|\hat{p}_i - \hat{p}_i^*\|_{\infty}^2 \le \|\hat{p}_i - \hat{p}_i^*\|_2^2 \le O\left(\frac{\beta \sigma(-2\beta)^{d_2}}{2^{d_2/2}(1 - \tanh^2(\beta))^2} \sqrt{\frac{\log(M/\delta)}{M}}\right) \]

Suppose $\delta = \delta'/n$, then by applying this argument at every node $i$ and taking the union bound we get the desired bound, with probability $1 - \delta'$, for $\hat{p}$ given by taking the coefficient of $x_S$ to equal (an arbitrary choice of) the matching coefficient of $x_{S \setminus \{i\}}$ in $\hat{p}_i$ for some $i \in S$.
\end{proof}
The following example shows that a bounded-degree assumption is necessary
for a result like Theorem~\ref{thm:low-dimensional-regression} to hold,
even without hidden nodes in the model.
\begin{example}[Lower bound for recovering external field in dense models]\label{example:masked-external-field}
We give an example of two Ising models on $n$ nodes which have very different external
fields but require $\Omega(\sqrt{n})$ samples to distinguish.
In model A, nodes $1, \ldots, n - 1$ have external field $1$,
node $n$ has external field $0$, and for $1 \le i \le n - 1$ there is an edge
of weight $1/n$ from node $i$ to node $n$.
Model B has no edges; it is a product distribution which moment matches the first
moments of model A. A straightforward calculation as in \cite{RavikumarWL}
shows that the symmetric KL between these models is $O(1/n)$ which implies the desired sample lower bound
(by tensorization of symmetric KL and Pinsker's inequality). However, in model A there is no external
field on node $n$ whereas in model B the external field is $\Omega(1)$.

This example also shows why --- even when we know the ground truth is sparse ---
methods that use the $\ell_1$-norm as a proxy for sparsity (like \cite{KlivansM,Vuffray}) may require many samples to learn the true external field; the coefficients in examples A and B have similar $\ell_1$-norm.
\end{example}
\begin{remark}
Consider an RBM with the property that for any two hidden
nodes, the neighborhood of one is never contained in the neighborhood
of the other (i.e. the neighborhoods are the ``opposite'' of a laminar family).
Suppose also that each node has a ``generic'' positive external field, so
the coefficient of the corresponding maximal monomial is bounded away from zero. 
Then by using the above guarantee for learning the MRF potential, thresholding
small coefficients and then looking at the maximal nonzero monomials it is possible
to recover the location of each of the hidden nodes.
\end{remark}

\subsection{Learning Ferromagnetic Ising Models with Arbitrary Latent Variables}
In this subsection we show how our learning algorithms can be generalized beyond the RBM setting to ferromagnetic
Ising models with an arbitrary set of hidden nodes \---- i.e. the interaction matrix can connect pairs of observed nodes and pairs of hidden nodes too. The marginal distribution on the observed nodes still
induces a Markov Random Field, although it no longer has as simple
a closed form as in Lemma~\ref{lem:rbm-as-mrf}.

In this setting, our goal is to learn the (induced) \emph{Markov blanket}
of every observed node $i$, which we continue to denote by $\mathcal{N}_2(i)$,
and we let $d_2$ denote the maximum size of $\mathcal{N}_2(i)$ among all observed
nodes $i$. The only
new ingredient we need is the following generalization of 
Lemma~\ref{lem:2hop-lb}:
\begin{lem}\label{lem:general-n2-lbd}
Suppose $i$ and $j$ are nodes in an $(\alpha,\beta)$-nondegenerate ferromagnetic Ising model. Suppose
$S \subset [n]$ is a set of nodes which \emph{do not} separate $i$ and $j$: 
then
\[ I_i(S \cup \{j\}) - I_i(S) \ge 2\sigma(-2 \beta) \alpha^{k} (1 - \tanh^2(\beta))^{k}. \]
where $k$ is the length of the shortest path from $i$ to $j$ which does
not go through $S$.
\end{lem}
\begin{proof}
Suppose that $v_1, \ldots, v_k$ is the path from $i$ to $j$ so $v_1 = i$
and $v_k = j$. Then by submodularity it suffices to prove the lower bound when
$S = [n] \setminus \{v_1,\ldots,v_k\}$. Since
\begin{align*}
\Pr(X_i = 1 | X_S = 1^S)
&=  \Pr(X_i = 1 | X_j = 1, X_S = 1^S) \Pr(X_j = 1 | X_S = 1^S) \\
&\quad + \Pr(X_i = 1 | X_j = -1, X_S = 1^S) \Pr(X_j = -1 | X_S = 1^S)
\end{align*}
and $I_i(S) = 2\Pr(X_i = 1 | X_S = 1^S) - 1$ and
$I_i(S \cup \{j\}) = 2\Pr(X_i = 1 | X_j = 1, X_S = 1^S) - 1$,
we see
\begin{align*}
&\frac{1}{2}(I_i(S \cup \{j\}) - I_i(S)) \\
&\quad= \Pr(X_j = -1 | X_S = 1^S)(\Pr(X_i = 1 | X_j = 1, X_S = 1^S) - \Pr(X_i = 1 | X_j = -1, X_S = 1^S)) \\
&\quad\ge \sigma(-2 \beta) (\Pr(X_i = 1 | X_j = 1, X_S = 1^S) - \Pr(X_i = 1 | X_j = -1, X_S = 1^S))
\end{align*}
by Lemma~\ref{lem:spin-freedom}. Conditioned on $X_S = 1^S$, the Ising model we are considering reduces to an Ising model on a linear graph, so applying the below Lemma~\ref{lem:line-graph-influence} proves the result.
\end{proof}
\begin{lemma}\label{lem:line-graph-influence}
Let $X_1, \ldots, X_n$ be the spins on an $(\alpha,\beta)$-nondegenerate ferromagnetic Ising
model on a linear graph with vertices labeled in order as $1$ to $n$. Then
\[ \Pr(X_1 = 1 | X_n = 1) - \Pr(X_1 = 1 | X_n = -1) \ge (\alpha(1 - \tanh^2(\beta)))^{n - 1} \]
\end{lemma}
\begin{proof}
We prove this by induction on $n$. When $n = 1$ the difference is clearly $1$. In general, using that $X_1,X_n$ are conditionally independent given $X_{n - 1}$ we see
\begin{align*}
&\Pr(X_1 = 1 | X_n = 1) - \Pr(X_1 = 1 | X_n = -1)  \\
&= \Pr(X_1 = 1 | X_{n - 1} = 1)(\Pr(X_{n - 1} = 1 | X_n = 1) - \Pr(X_{n - 1} = 1 | X_{n - 1} = -1)) \\
&\quad + \Pr(X_1 = 1 | X_{n - 1} = -1)(\Pr(X_{n - 1} = -1 | X_n = 1) - \Pr(X_{n - 1} = -1 | X_{n - 1} = -1)) \\
&= (\Pr(X_1 = 1 | X_{n - 1} = 1) - \Pr(X_1 = 1 | X_{n - 1} = -1)) \\
&\quad \cdot (\Pr(X_{n - 1} = 1 | X_n = 1) - \Pr(X_{n - 1} = 1 | X_{n - 1} = -1)) \\
&\ge (\alpha(1 - \tanh^2(\beta)))^{n - 1}
\end{align*}
by the induction hypothesis and Lemma~\ref{lem:tanh-sensitivity}
\end{proof}
As in the RBM case, Lemma~\ref{lem:general-n2-lbd} shows in particular
that $\mathcal{N}_2(i)$ equals its obvious graph-theoretic analogue: the set
of nodes $j$ such that $i$ and $j$ are connected by a path whose intermediate nodes are all latent. We also get the following natural generalization of  Theorem~\ref{thm:greedynbhd-works} for recovering $\mathcal{N}_2(i)$:
\begin{theorem}\label{thm:greedynbhd-works-general}
  Let $\delta > 0$. Suppose $X^{(1)}, \ldots, X^{(M)}$ are samples from the observable distribution of an Ising model with hidden nodes which is $(\alpha,\beta)$-nondegenerate. Suppose also that $d_2$ is known such that $d_2 \ge |\mathcal{N}_2(i)|$ for all observed nodes $i$ and
  that for every $i$ and $j \in \mathcal{N}_2(i)$, there is a path of length at most $\ell$ from $i$ to $j$. Then if
  \[ M \ge 2^{2k + 3} (d_2/\eta)^2 (\log(n) + k\log(en/k)) \log(4/\delta) \]
  where we set
  \[ \eta = \alpha^{\ell} \sigma(-2\beta)(1 - \tanh(\beta))^{\ell}, \qquad k = d_2 \log(4/\eta), \] for every $i$ algorithm {\sc GreedyNbhd} returns $\mathcal{N}_2(i)$, with probability at least $1 - \delta$. Furthermore the total runtime is $O(M k n^2) = e^{O(\beta \ell d_2 - \ell \log(\alpha))} n^2 \log(n)$.
\end{theorem}
\begin{proof}
This is the same as proof of Theorem~\ref{thm:greedynbhd-works}, except
that we replace the use of Lemma~\ref{lem:2hop-lb} by Lemma~\ref{lem:general-n2-lbd}.
\end{proof}
The corresponding analogue of Theorem~\ref{thm:searchnbhd-works} follows as well by using Lemma~\ref{lem:general-n2-lbd}. Once the neighborhood structure is determined,
we can again learn the MRF potential in a straightforward way by regression.
\begin{theorem}\label{thm:low-dimensional-regression-general}
Consider an unknown $(\alpha,\beta)$-nondegenerate Ising model with hidden nodes and $d_2,\ell$ as in statement of Theorem~\ref{thm:greedynbhd-works-general}. Let $p^*(x)$ be the potential of the MRF induced on the observed nodes.
There exists an algorithm, which given $\alpha,\beta,d_2,\ell$,
with probability at least $1 - \delta$ 
finds a polynomial $p$ of degree at most $d_2$ such that 
\[ \|\hat{p} - \hat{p^*}\|_{\infty}^2 = O\left(\frac{\beta \sigma(-2\beta)^{d_2}}{2^{d_2/2}(1 - \tanh^2(\beta))^2} \sqrt{\frac{\log(Mn/\delta)}{M}}\right) \]
given $M$ samples from the distribution on observed nodes provided that $M$ is at the required $M$ in Theorem~\ref{thm:greedynbhd-works-general}.
\end{theorem}
\begin{proof}
Same as proof of Theorem~\ref{thm:low-dimensional-regression}, except we replace use of Theorem~\ref{thm:greedynbhd-works} by Theorem~\ref{thm:greedynbhd-works-general}.
\end{proof}
\section{Inference on the Induced MRF via the Lee-Yang Property}

We first recall various results from \cite{LSS}, whose approach is
based on Barvinok's approach \cite{Barvinok} for approximating the log-partition
function. The basic idea is to Taylor expand $\log Z$
around the point of infinite external field, where $\log Z$ is easy to
compute because only one spin configuration contributes. 
A \emph{Lee-Yang property}\footnote{Here we are following the terminology of \cite{LSS}. There is an unrelated ``Lee-Yang property'' which appears in the literature on Lee-Yang for general real-valued spins.} can be used to prove
that the Taylor expansion is accurate. 

\begin{defn}[Lee-Yang property]
  Let $P(z_1, \ldots, z_n)$ be a multilinear polynomial with real
  coefficients.  $P$ has the \emph{Lee-Yang property} if for any
  choice of complex numbers $\lambda_1, \ldots, \lambda_n$ such that 
  $|\lambda_i| \le 1$ for all $i$ and $|\lambda_i| < 1$ for at least one $i$,
  we have that $P(\lambda_1, \ldots, \lambda_n) \ne 0$.
\end{defn}
Typically the polynomial $P$ arises as the partition function of a Markov
Random Field, where the $\lambda_i$ are a re-parameterization of the external field.
This is illustrated in the classical Lee-Yang theorem \cite{lee-yang}:
\begin{theorem}[Lee-Yang, \cite{lee-yang}]\label{thm:lee-yang}
  Suppose $J_{ij} \ge 0$ and
\[ P(\lambda_1, \ldots, \lambda_n) := \sum_{x \in \{\pm 1\}^n} \exp(\frac{1}{2} \sum_{i,j} J_{ij} x_i x_j) \prod_{i : x_i = 1} \lambda_i, \]
  so that $\left(\prod_{i = 1}^n \lambda_i^{-1/2}\right) P(\lambda_1, \ldots, \lambda_n)$
  for positive real $\lambda_i$ is the partition function of a ferromagnetic Ising
  model with external field $h_i = \frac{1}{2} \log \lambda_i$. Then $P$ extends
  to complex $\lambda_i$ as a multilinear polynomial with the
  Lee-Yang property.
\end{theorem}
The Lee-Yang property translates back to the following statement about the partition function:
\begin{corollary}\label{crly:lee-yang-halfplane}
Suppose $Z(h) = \sum_{x \in \{\pm 1\}^n} \exp(\frac{1}{2} \sum_{i,j} J_{ij} x_i x_j + \sum_i h_i x_i)$ is the partition function of a ferromagnetic Ising model with consistent non-positive external fields, i.e. $h_i \le 0$ for all $i$. If we extend $Z$ to complex $h$, then $Z(h) \ne 0$ for any $h$ with $\Re(h_i) \le 0$ for all $i$ and $\Re(h_i) < 0$ for at least one $i$.
\end{corollary}
\begin{proof}
This follows from the by taking $\lambda_i = e^{2 h_i}$ so 
\[ Z(h) = \left(\prod_{i = 1}^n \lambda_i^{-1/2}\right) P(\lambda_1, \ldots, \lambda_n) = e^{-\sum_{i = 1}^n h_i} P(\lambda_1, \ldots, \lambda_n), \] and using the non-vanishing of $P$ by the previous Theorem.
\end{proof}
As we see the $\lambda_i$ with $|\lambda_i| \le 1$ correspond to non-positive
external fields, whereas previously we assumed the external fields were non-negative.
However the partition function is invariant to the global sign flip $x \mapsto -x$ so this is equivalent; this choice is made so we expand $P$ around $0$ instead of $\infty$.
The following Lemma bounds the error made when we do this Taylor expansion.
\begin{lemma}[Lemma 2.1 of \cite{LSS}]\label{lem:lss-taylor-expansion}
  Suppose that
  \begin{equation}\label{eqn:z-lambda}
  Z(\lambda) = C \sum_{x \in \{\pm 1\}^n} \exp\left(\sum_{e \in E} f_e(x_e)\right) \lambda^{\#\{v : x_v = 1\}}
  \end{equation}
  where $E$ is the set of edges of a hypergraph and each $f_e$ is a real-valued function.   Suppose $0 < \epsilon < \frac{1}{4}$ and $$m \ge \frac{|\lambda|}{1 - |\lambda|}
  \left (\log(4n/\epsilon) + \log(\frac{1}{1 - |\lambda|})\right)$$ and the values
  of $\frac{d^j}{d\lambda_j} Z(\lambda)|_{\lambda = 0}$ are
  given for $j = 0, \ldots, m$. Finally, suppose the Lee-Yang property holds for
  $Z(\lambda)$ as a univariate polynomial. Then for any $\lambda$ with $|\lambda| < 1$, 
  there is an algorithm which computes an additive
  $\epsilon/4$-approximation to $\log Z(\lambda)$ in polynomial time.
\end{lemma}
This lemma does not specify a way to compute the needed values
of $\frac{d^j}{d\lambda^j} Z(\lambda)|_{\lambda = 0}$. However, for $j = 0$
this is easy to compute, because the only non-zero in the sum is when $x$
is the all-1s vector. For $j \ge 1$, this is provided
by Theorem~3.1 of \cite{LSS} (building on the work of \cite{patel2017deterministic}) as long as the underlying hypergraph of the MRF has bounded degree. Recall that the \emph{degree} of a vertex in a hypergraph is the number of hyperedges containing it.
\begin{theorem}[Theorem 3.1 of \cite{LSS}]\label{thm:lss-insects}
  Fix $C > 0, d \in \mathbb{N}$. Suppose we are given as input
  an $n$-vertex hypergraph with edge set
  $E$ of maximum degree $d$ and maximum hyperedge size $r$, 
  and $Z(\lambda)$ is defined as in \eqref{eqn:z-lambda}.
  Then for any $\epsilon > 0$ there exists a deterministic $poly_{C,d,r}(n/\epsilon)$
  time algorithm to compute $\frac{d^j}{d\lambda^j} Z(\lambda)|_{\lambda = 0}$
  for $j = 1, \ldots, m$ where $m = \lceil C \log(n/\epsilon) \rceil$.
\end{theorem}
Finally, we describe how to apply these results to sample
from the MRF induced by an RBM. The key is that,
from the proof of Lemma~\ref{lem:rbm-as-mrf}, we see that
the induced MRF has the same partition function as the original
Ising model, so it inherits the Lee-Yang property guaranteed
by Theorem~\ref{thm:lee-yang}:
\begin{lemma}\label{lem:lee-yang-induced}
Fix a ferromagnetic RBM with consistent non-positive external fields on the hidden
nodes (i.e. $h^{(2)}_i \le 0$) and with external field $h^{(1)}_i := h^0_i + h_i$ with $h^0_i,h_i \le 0$
on observed node $i$.
Hence (by Lemma~\ref{lem:rbm-as-mrf}) the induced MRF has potential $g(x) + h \cdot x$ for some polynomial $g : \{ \pm 1\}^n \to \mathbb{R}$ not depending on $h$, such that
\[ \Pr(X = x) = \frac{1}{Z(h_0 + h)}\exp(g(x) + h \cdot x) \]
for $x \in \{ \pm 1\}^n$ where $Z(h^0 + h)$ is the partition function of the RBM. Let
\[ P(\lambda_1, \ldots, \lambda_n) :=  \sum_{x \in \{\pm 1\}^n} \exp(g(x)) \prod_{i : x_i = 1} \lambda_i. \]
Then $P$ has the Lee-Yang property.
\end{lemma}
\begin{proof}
As before, we see that if $\lambda_i = e^{2 h_i}$ then
\[ Z(h_0 + h) = \left(\prod_{i = 1}^n \lambda_i^{-1/2}\right) P(\lambda_1, \ldots, \lambda_n) = e^{-\sum{i = 1}^n h_i} P(\lambda_1, \ldots, \lambda_n). \]
We prove the theorem by induction on $n$, the number of observed nodes. If all of the $\lambda_i$ equal 0 then it is
clear that $P \ne 0$ as the sum is over only a single non-zero term. If there is at least one $\lambda_i$ such that $\lambda_i = 0$, then $P(\lambda_1, \ldots, \lambda_n)$ agrees with the $P$ associated to the smaller RBM formed by conditioning $X_i = -1$, hence the non-vanishing follows by the induction hypothesis. Otherwise if all of the $\lambda_i$ are non-zero, then
we know by Corollary~\ref{crly:lee-yang-halfplane} that $Z(h_0 + h) \ne 0$ and we deduce that $P(\lambda) \ne 0$.
\end{proof}
Combining
these results, we obtain the following theorem:
\begin{theorem}\label{thm:lee-yang-logz}
  Fix $H > 0$ and a maximum degree $d_2$. Then for any ferromagnetic RBM in which 
  $h^{(1)}_i \le -H$ for all $i$, there
  is a deterministic polynomial time algorithm which given
  any $0 < \epsilon < 1/4$ and the description of the induced MRF,
  computes $\log Z$ within
  additive error $\epsilon/4$ where $Z$ is the partition function of the induced MRF.
\end{theorem}
\begin{proof}
By assumption we know a function $f$ such that
\[ \Pr(X = x) = \frac{1}{Z_f} \exp(f(x)). \]
If we take
$\omega^*$ such that $\frac{1}{2} \log \omega^* = -H$, we see
\[ Z_{f} = \left(\prod_{i = 1}^n \omega^* \right)^{-1/2} Q(\omega^*,\ldots, \omega^*) \]
where
\[ Q(\omega_i) =  \sum_{x \in \{\pm 1\}^n} \exp(f(x) + H \sum x_i) \prod_{i : x_i = 1} \omega_i . \]
Comparing $f$ and $Q$ to $g$ and $P$ from Lemma~\ref{lem:lee-yang-induced},
which we apply with $h^0 = h^{(1)} + H$, we see
that $Q$ differs from $P$ only by a multiplicative constant (corresponding to $e^{\hat{g}(\emptyset) - \hat{f}(\emptyset)}$) so $Q$ also has the Lee-Yang property. 
Therefore we can compute $Q(\omega^*,\ldots,\omega^*)$ and so $Z_f$ efficiently
by the results of Lemma~\ref{lem:lee-yang-induced} and Theorem~\ref{thm:lss-insects}.
\end{proof}
The significance of accurately estimating $\log Z$ is that 
it allows for the performance of various inference tasks
which are otherwise computationally intractable. For example, we can estimate to high precision the likelihood of observing any particular output from the MRF, since
\[ \log \Pr(X = x) = p(x) - \log Z,\]
where $p(x)$ is the potential of the MRF.
Hence the $\epsilon/4$ approximation to $\log Z$ from Theorem~\ref{thm:lee-yang-logz} implies an $\epsilon/4$
approximation to $\log \Pr(X = x)$, i.e. a PTAS for estimating $\Pr(X = x)$.

\vspace{.2cm}
\noindent
\textbf{Acknowledgements: } We are grateful to Elchanan Mossel for valuable discussions about this work and to Linus Hamilton for valuable preliminary discussions. We also thank Raghu Meka for useful feedback.
\bibliographystyle{plain}
\bibliography{bib}

\begin{thebibliography}{10}

\bibitem{Anandkumar}
Animashree Anandkumar, Vincent~YF Tan, Furong Huang, and Alan~S Willsky.
\newblock High-dimensional structure estimation in ising models: Local
  separation criterion.
\newblock {\em The Annals of Statistics}, pages 1346--1375, 2012.

\bibitem{AnandkumarV}
Animashree Anandkumar and Ragupathyraj Valluvan.
\newblock Learning loopy graphical models with latent variables: Efficient
  methods and guarantees.
\newblock {\em The Annals of Statistics}, pages 401--435, 2013.

\bibitem{Arora}
Sanjeev Arora, Aditya Bhaskara, Rong Ge, and Tengyu Ma.
\newblock Provable bounds for learning some deep representations.
\newblock In {\em International Conference on Machine Learning}, pages
  584--592, 2014.

\bibitem{Barvinok}
Alexander Barvinok.
\newblock Computing the permanent of (some) complex matrices.
\newblock {\em Foundations of Computational Mathematics}, 16(2):329--342, 2016.

\bibitem{BogdanovMV}
Andrej Bogdanov, Elchanan Mossel, and Salil Vadhan.
\newblock The complexity of distinguishing markov random fields.
\newblock In {\em Approximation, Randomization and Combinatorial Optimization.
  Algorithms and Techniques}, pages 331--342. Springer, 2008.

\bibitem{Bresler}
Guy Bresler.
\newblock Efficiently learning ising models on arbitrary graphs.
\newblock In {\em Proceedings of the Forty-Seventh Annual ACM on Symposium on
  Theory of Computing}, pages 771--782. ACM, 2015.

\bibitem{bresler2014hardness}
Guy Bresler, David Gamarnik, and Devavrat Shah.
\newblock Hardness of parameter estimation in graphical models.
\newblock In {\em Advances in Neural Information Processing Systems}, pages
  1062--1070, 2014.

\bibitem{BreslerMosselSly}
Guy Bresler, Elchanan Mossel, and Allan Sly.
\newblock Reconstruction of markov random fields from samples: Some
  observations and algorithms.
\newblock In {\em Approximation, Randomization and Combinatorial Optimization.
  Algorithms and Techniques}, pages 343--356. Springer, 2008.

\bibitem{multilinear}
Gruia Calinescu, Chandra Chekuri, Martin P{\'a}l, and Jan Vondr{\'a}k.
\newblock Maximizing a submodular set function subject to a matroid constraint.
\newblock In {\em International Conference on Integer Programming and
  Combinatorial Optimization}, pages 182--196. Springer, 2007.

\bibitem{ChandrasekaranPW}
Venkat Chandrasekaran, Pablo~A Parrilo, and Alan~S Willsky.
\newblock Latent variable graphical model selection via convex optimization.
\newblock {\em The Annals of Statistics}, 40(4):1935--1967, 2012.

\bibitem{ChowL}
C~Chow and Cong Liu.
\newblock Approximating discrete probability distributions with dependence
  trees.
\newblock {\em IEEE transactions on Information Theory}, 14(3):462--467, 1968.

\bibitem{RBMfeature}
Adam Coates, Andrew Ng, and Honglak Lee.
\newblock An analysis of single-layer networks in unsupervised feature
  learning.
\newblock In {\em Proceedings of the fourteenth international conference on
  artificial intelligence and statistics}, pages 215--223, 2011.

\bibitem{eldan2016power}
Ronen Eldan and Ohad Shamir.
\newblock The power of depth for feedforward neural networks.
\newblock In {\em Conference on Learning Theory}, pages 907--940, 2016.

\bibitem{Goel}
Surbhi Goel and Adam Klivans.
\newblock Learning depth-three neural networks in polynomial time.
\newblock {\em arXiv preprint arXiv:1709.06010}, 2017.

\bibitem{ghs}
Robert~B. Griffiths, C.~A. Hurst, and S.~Sherman.
\newblock Concavity of magnetization of an ising ferromagnet in a positive
  external field.
\newblock {\em Journal of Mathematical Physics}, 11(3):790--795, 1970.

\bibitem{Hamilton}
Linus Hamilton, Frederic Koehler, and Ankur Moitra.
\newblock Information theoretic properties of markov random fields, and their
  algorithmic applications.
\newblock In {\em Advances in Neural Information Processing Systems}, pages
  2460--2469, 2017.

\bibitem{Hinton}
Geoffrey~E Hinton.
\newblock Deep belief networks.
\newblock {\em Scholarpedia}, 4(5):5947, 2009.

\bibitem{RBMdimension}
Geoffrey~E Hinton and Ruslan~R Salakhutdinov.
\newblock Reducing the dimensionality of data with neural networks.
\newblock {\em science}, 313(5786):504--507, 2006.

\bibitem{RBMtopic}
Geoffrey~E Hinton and Ruslan~R Salakhutdinov.
\newblock Replicated softmax: an undirected topic model.
\newblock In {\em Advances in neural information processing systems}, pages
  1607--1614, 2009.

\bibitem{Janzamin}
Majid Janzamin, Hanie Sedghi, and Anima Anandkumar.
\newblock Beating the perils of non-convexity: Guaranteed training of neural
  networks using tensor methods.
\newblock {\em arXiv preprint arXiv:1506.08473}, 2015.

\bibitem{JerrumSinclair:90}
M.~Jerrum and A.~Sinclair.
\newblock Polynomial-time approximation algorithms for ising model (extended
  abstract).
\newblock In {\em Automata, Languages and Programming}, pages 462--475, 1990.

\bibitem{glmtron}
Sham~M Kakade, Varun Kanade, Ohad Shamir, and Adam Kalai.
\newblock Efficient learning of generalized linear and single index models with
  isotonic regression.
\newblock In {\em Advances in Neural Information Processing Systems}, pages
  927--935, 2011.

\bibitem{KargerS}
David Karger and Nathan Srebro.
\newblock Learning markov networks: Maximum bounded tree-width graphs.
\newblock In {\em Proceedings of the twelfth annual ACM-SIAM symposium on
  Discrete algorithms}, pages 392--401. Society for Industrial and Applied
  Mathematics, 2001.

\bibitem{KlivansM}
Adam Klivans and Raghu Meka.
\newblock Learning graphical models using multiplicative weights.
\newblock In {\em FOCS}, 2017.

\bibitem{lee-yang}
Tsung-Dao Lee and Chen-Ning Yang.
\newblock Statistical theory of equations of state and phase transitions. ii.
  lattice gas and ising model.
\newblock {\em Physical Review}, 87(3):410, 1952.

\bibitem{LSS}
Jingcheng Liu, Alistair Sinclair, and Piyush Srivastava.
\newblock The ising partition function: Zeros and deterministic approximation.
\newblock {\em 2017 IEEE 58th Annual Symposium on Foundations of Computer
  Science (FOCS)}, pages 986--997, 2017.

\bibitem{LongS}
Philip~M Long and Rocco Servedio.
\newblock Restricted boltzmann machines are hard to approximately evaluate or
  simulate.
\newblock In {\em Proceedings of the 27th International Conference on Machine
  Learning (ICML-10)}, pages 703--710, 2010.

\bibitem{lynn2018maximizing}
Christopher~W Lynn and Daniel~D Lee.
\newblock Maximizing activity in ising networks via the tap approximation.
\newblock {\em arXiv preprint arXiv:1803.00110}, 2018.

\bibitem{martens-rbm-representation}
James Martens, Arkadev Chattopadhya, Toni Pitassi, and Richard Zemel.
\newblock On the representational efficiency of restricted boltzmann machines.
\newblock In {\em Advances in Neural Information Processing Systems}, pages
  2877--2885, 2013.

\bibitem{MOS}
Elchanan Mossel, Ryan O'Donnell, and Rocco~P Servedio.
\newblock Learning juntas.
\newblock In {\em Proceedings of the thirty-fifth annual ACM symposium on
  Theory of computing}, pages 206--212. ACM, 2003.

\bibitem{greedysubmodular}
George~L Nemhauser, Laurence~A Wolsey, and Marshall~L Fisher.
\newblock An analysis of approximations for maximizing submodular set
  functions—i.
\newblock {\em Mathematical Programming}, 14(1):265--294, 1978.

\bibitem{patel2017deterministic}
Viresh Patel and Guus Regts.
\newblock Deterministic polynomial-time approximation algorithms for partition
  functions and graph polynomials.
\newblock {\em SIAM Journal on Computing}, 46(6):1893--1919, 2017.

\bibitem{RavikumarWL}
Pradeep Ravikumar, Martin~J Wainwright, John~D Lafferty, et~al.
\newblock High-dimensional ising model selection using ?1-regularized logistic
  regression.
\newblock {\em The Annals of Statistics}, 38(3):1287--1319, 2010.

\bibitem{safran2017depth}
Itay Safran and Ohad Shamir.
\newblock Depth-width tradeoffs in approximating natural functions with neural
  networks.
\newblock In {\em International Conference on Machine Learning}, pages
  2979--2987, 2017.

\bibitem{RBMcollaborative}
Ruslan Salakhutdinov, Andriy Mnih, and Geoffrey Hinton.
\newblock Restricted boltzmann machines for collaborative filtering.
\newblock In {\em Proceedings of the 24th international conference on Machine
  learning}, pages 791--798. ACM, 2007.

\bibitem{santhanamW}
Narayana~P Santhanam and Martin~J Wainwright.
\newblock Information-theoretic limits of selecting binary graphical models in
  high dimensions.
\newblock {\em IEEE Transactions on Information Theory}, 58(7):4117--4134,
  2012.

\bibitem{SlySun}
Allan Sly and Nike Sun.
\newblock The computational hardness of counting in two-spin models on
  d-regular graphs.
\newblock In {\em Foundations of Computer Science (FOCS), 2012 IEEE 53rd Annual
  Symposium on}, pages 361--369. IEEE, 2012.

\bibitem{telgarsky2016benefits}
Matus Telgarsky.
\newblock Benefits of depth in neural networks.
\newblock In {\em Conference on Learning Theory}, pages 1517--1539, 2016.

\bibitem{Valiant}
Gregory Valiant.
\newblock Finding correlations in subquadratic time, with applications to
  learning parities and juntas.
\newblock In {\em Foundations of Computer Science (FOCS), 2012 IEEE 53rd Annual
  Symposium on}, pages 11--20. IEEE, 2012.

\bibitem{Vuffray}
Marc Vuffray, Sidhant Misra, Andrey Lokhov, and Michael Chertkov.
\newblock Interaction screening: Efficient and sample-optimal learning of ising
  models.
\newblock In {\em Advances in Neural Information Processing Systems}, pages
  2595--2603, 2016.

\bibitem{Zhang}
Yuchen Zhang, Jason~D Lee, and Michael~I Jordan.
\newblock $\ell_1$-regularized neural networks are improperly learnable in
  polynomial time.
\newblock In {\em International Conference on Machine Learning}, pages
  993--1001, 2016.

\end{thebibliography}
\end{document}